\documentclass[11pt]{article}
\usepackage{enumerate}
\usepackage{pdfsync}
\usepackage[OT1]{fontenc}

\usepackage[usenames]{color}
\usepackage{smile}
\usepackage[colorlinks,
            linkcolor=red,
            anchorcolor=blue,
            citecolor=blue
            ]{hyperref}
\usepackage{fullpage}
\usepackage{hyperref}
\usepackage[protrusion=true,expansion=true]{microtype}
\usepackage{pbox}
\usepackage{setspace}
\usepackage{tabularx}
\usepackage{float}
\usepackage{wrapfig,lipsum}
\usepackage{enumitem}
\usepackage{microtype}
\usepackage{graphicx}
\usepackage{subfigure}
\usepackage{enumerate}
\usepackage{enumitem}
\usepackage{pgfplots}
\usepackage{mathtools}
\usepackage{caption}
\usetikzlibrary{arrows,shapes,snakes,automata,backgrounds,petri}
\usepackage{booktabs} % for professional tables

\allowdisplaybreaks
\usepackage{colortbl}
\definecolor{LightCyan}{rgb}{0.8, 0.9, 1}
\definecolor{LightGray}{gray}{0.9}

\usepackage{enumitem}
\usepackage{colortbl}
\definecolor{LightCyan}{rgb}{0.8, 0.9, 1}
\newcommand{\alglinelabel}{%
  \addtocounter{ALC@line}{-1}% Reduce line counter by 1
  \refstepcounter{ALC@line}% Increment line counter with reference capability
  \label% Regular \label
}
\usepackage{xcolor}
\usepackage{adjustbox}

\ifdefined\final
\usepackage[disable]{todonotes}
\else
\usepackage[textsize=tiny]{todonotes}
\fi
\makeatletter
\newcommand*{\rom}[1]{\expandafter\@slowromancap\romannumeral #1@}
\makeatother
\title{\huge Variance-Dependent Regret Bounds for Non-stationary Linear Bandits}
\author
{
    Zhiyong Wang\thanks{The Chinese University of Hong Kong; {\tt zhiyongwangwzy@gmail.com}}
    ~~~~
    Jize Xie\thanks{Hong Kong University of Science and Technology; {\tt jxiebj@connect.ust.hk}}
    ~~~~
    Yi Chen\thanks{Hong Kong University of Science and Technology; {\tt yichen@ust.hk}} 
    ~~~~
    John C.S. Lui\thanks{The Chinese University of Hong Kong; {\tt cslui@cse.cuhk.edu.hk}}
    ~~~~
    Dongruo Zhou\thanks{Indiana University Bloomington; {\tt dz13@iu.edu}}
}
\begin{document}
    \date{}
    \maketitle

\begin{abstract}

We investigate the non-stationary stochastic linear bandit problem where the reward distribution evolves each round. Existing algorithms characterize the non-stationarity by the \emph{total variation budget} $B_K$, which is the summation of the change of the consecutive feature vectors of the linear bandits over $K$ rounds. However, such a quantity only measures the non-stationarity with respect to the expectation of the reward distribution, which makes existing algorithms sub-optimal under the general non-stationary distribution setting. In this work, we propose algorithms that utilize the 
\emph{variance} of the reward distribution as well as the $B_K$, and show that they can achieve tighter regret upper bounds. Specifically, we introduce two novel algorithms: Restarted Weighted$\text{OFUL}^+$ and Restarted $\text{SAVE}^+$. These algorithms address cases where the variance information of the rewards is known and unknown, respectively. Notably, when the total variance $V_K$ is much smaller than $K$, our algorithms outperform previous state-of-the-art results on non-stationary stochastic linear bandits under different settings. Experimental evaluations further validate the superior performance of our proposed algorithms over existing works.
\end{abstract}

\section{Introduction}
In this work, we study non-stationary stochastic bandits, which is a generalization of the classical stationary stochastic bandits, where the reward distribution is non-stationary. The intuition about the non-stationary setting comes from real-world applications such as dynamic pricing and ads allocation, where the environment changes rapidly and deviates significantly from stationarity \cite{auer2002nonstochastic,cheung2018hedging}. Most of the existing works in stochastic bandits consider a stationary setting where the goal of the agent is to minimize the \emph{static regret}, \emph{i.e.}, the summation of suboptimality gaps between the agent's selected arm and the fixed, time-independent best arm that maximizes the expectation of the reward distribution. In contrast, for the non-stationary setting, the emphasis shifts to minimizing the \emph{dynamic regret}, which represents the gap between the cumulative reward of selecting the time-dependent optimal arm at each time and that of the learner. As we can always treat a stationary bandit instance as a special case of the non-stationary bandit instance, designing algorithms that work well under the non-stationary setting is significantly more challenging.  

There have been a series of works aiming to minimize the \emph{dynamic regret} for non-stationary stochastic bandits, such as Multi-Armed Bandits (MAB) \citep{auer2002nonstochastic,10.1007/978-3-642-24412-4_16,besbes2014stochastic,wei2016tracking}, linear bandits \citep{cheung2018hedging,cheung2019learning,zhao2020simple,wei2021non,wang2023revisiting}, general function approximation \citep{faury2021regret,russac2020algorithms,russac2021self}, and the even more challenging reinforcement learning (RL) setting \citep{pmlr-v139-mao21b,touati2020efficient,gajane2018sliding,cheung2020reinforcement,wei2021non}. In this work, we mainly consider the linear bandit setting, where each arm is a contextual vector, and the expected reward of each arm is assumed to be the linear product of the arm with an unknown feature vector. Most existing \emph{dynamic regret} results for non-stationary linear bandits depend on both the \emph{non-stationarity measurement} and the number of interaction rounds. Specifically, assume $K$ is the total number of rounds in bandits, and for each $k\in[K]$, $\xb$ is one of the arms, $\btheta_k$ and $\btheta_{k+1}$ are the feature vectors at $k$ and $k+1$ rounds, satisfying $\|\xb\|_2 \leq 1$. Then, the non-stationarity measurement is often defined as the summation of the changes in the mean of the reward distribution, which is
\begin{align}
    B_K:=\sum_{k=1}^K \max_{\xb \in \RR^d}|\la\xb, \btheta_k - \btheta_{k+1}\ra|= \sum_{k=1}^K\|\btheta_k - \btheta_{k+1}\|_2\,.
\end{align}
Existing works for non-stationary linear bandits \citep{russac2019weighted, kim2019randomized, pmlr-v108-zhao20a, touati2020efficient,cheung2018hedging,zhao2020simple} achieved a regret upper bound of $\Tilde{O}(d^{7/8}B_K^{\frac{1}{4}}K^{\frac{3}{4}})$, where $d$ is the problem dimension. A recent work by \cite{wei2021non} proposed a black-box reduction method that can achieve a regret upper bound of $\Tilde{O}(dB_K^{\frac{1}{3}}K^{\frac{2}{3}})$
in the setting with a fixed arm set across all rounds. Such regret bounds clearly demonstrate that the regret grows as long as the non-stationarity grows, which is aligned with intuition. 

Although existing works clearly demonstrate the relationship between the $B_K$ and the regret, we claim that it is not sufficient for us to fully characterize the non-stationary level of the reward distributions. Consider applications such as hyperparameter tuning in physical systems, the noise distribution may highly depend on the evaluation point since the measurement noise often
largely varies with the chosen parameter settings \cite{kirschner2018information}. 
For linear bandits, such examples suggest that the non-stationarity not only consists of the change of the mean of the distribution, but also the variance of the distribution. However, none of the previous works on non-stationary linear bandits considered how to leverage the variance information to improve regret bounds in the above heteroscedastic noise setting. 
Therefore, an open question arises:\begin{center}
    \emph{Can we design even better algorithms for non-stationary linear bandits by considering its variance information?}
    \end{center}

In this paper, we answer this question affirmatively. We assume that at the $k$-th round, the reward distribution of an arm $\xb$ satisfies $r_k \sim \la \btheta_k, \xb\ra + \epsilon_k$, where $\epsilon_k$ is a zero-mean noise variable with variance $\sigma_k^2$. Our contributions are:
\begin{itemize}[leftmargin =*]
        \item For the case where the reward variance $\sigma_k^2$ at round $k$ can be observed and the \emph{total variation budget} $B_K$ is known, we propose the Restarted-$\text{WeightedOFUL}^+$ algorithm, which uses variance-based weighted linear regression to deal with heteroscedastic noises \citep{zhou2021nearly, zhou2022computationally}    
        and a restarted scheme to forget some historical data to hedge against the non-stationarity. We prove that the regret upper bound of Restarted-$\text{WeightedOFUL}^+$ is $\Tilde{O}(d^{7/8}(B_KV_K)^{1/4}\sqrt{K} + d^{5/6}B_K^{1/3}K^{2/3})$. Notably, our regret surpasses the best result for non-stationary linear bandits $\tilde O(dB_K^{1/3}K^{2/3})$ \citep{wei2021non} when the total variance $V_K = \tilde O(1)$ is small, which indicates that additional variance information benefits non-stationary linear bandit algorithms. 
         \item For the case where the reward variance $\sigma_k^2$ is unknown but the total variance $V_K$ and variation budget $B_K$ are known, we propose the Restarted-$\text{SAVE}^+$ algorithm. It maintains a multi-layer weighted linear regression structure with carefully-designed weight within each layer to handle the unknown variances \citep{zhao2023variance}. We prove that Restarted-$\text{SAVE}^+$ can achieve a regret upper bound of $\tilde O(d^{\frac{4}{5}}V_K^{\frac{2}{5}}B_K^{\frac{1}{5}}K^{\frac{2}{5}}+ d^{\frac{2}{3}}B_K^{\frac{1}{3}}K^{\frac{2}{3}})$. Specifically, when $V_K = \tilde O(1)$, our regret is also better than the existing best result $\tilde O(dB_K^{1/3}K^{2/3})$ \citep{wei2021non}, which again verifies the effect of the variance information.
        \item Lastly, we propose Restarted-$\text{SAVE}^+$-BOB for the case where both the reward variance $\sigma_k^2$ and $B_K$ are unknown. Restarted-$\text{SAVE}^+$-BOB equips a \emph{bandit-over-bandit} (BOB) framework to handle the unknown $B_K$ \citep{cheung2019learning}, and also maintains a multi-layer structure as Restarted-$\text{SAVE}^+$. We show that Restarted-$\text{SAVE}^+$-BOB achieves a regret upper bound of $\tilde O(d^{\frac{4}{5}}V_K^{\frac{2}{5}}B_K^{\frac{1}{5}}K^{\frac{2}{5}}+ d^{\frac{2}{3}}B_K^{\frac{1}{3}}K^{\frac{2}{3}}+d^{\frac{1}{5}}K^{\frac{7}{10}})$, and it behaves the same as Restarted-$\text{SAVE}^+$ when $V_K = \tilde O(1)$ and $B_K = \Omega(d^{-14}K^{1/10})$.
        % \item Notably, our proposed algorithms apply to the contextual setting where the arm set varies over time. In the worst case with $\sum_{k=1}^K \sigma_k^2=O(K)$, our results are $\Tilde{O}(B_K^{\frac{1}{4}}K^{\frac{3}{4}})$, matching the state-of-the-art results in the contextual setting. In the case where $\sum_{k=1}^K \sigma_k^2$ is small, our results are $\Tilde{O}(B_K^{\frac{1}{3}}K^{\frac{2}{3}})$, better than previous results in contextual setting in terms of the dependence on $K$.
        \item We also conduct experimental evaluations to validate the outperformance of our proposed algorithms over existing works.
    \end{itemize}
    \paragraph{Notation} 
    We use lower case letters to denote scalars, and use lower and upper case bold face letters to denote vectors and matrices respectively. We denote by $[n]$ the set $\{1,\dots, n\}$. For a vector $\xb\in \RR^d$ and a positive semi-definite matrix $\bSigma\in \RR^{d\times d}$, we denote by $\|\xb\|_2$ the vector's Euclidean norm and define $\|\xb\|_{\bSigma}=\sqrt{\xb^\top\bSigma\xb}$. For two positive sequences $\{a_n\}$ and $\{b_n\}$ with $n=1,2,\dots$, 
    we write $a_n=O(b_n)$ if there exists an absolute constant $C>0$ such that $a_n\leq Cb_n$ holds for all $n\ge 1$ and write $a_n=\Omega(b_n)$ if there exists an absolute constant $C>0$ such that $a_n\geq Cb_n$ holds for all $n\ge 1$. We use $\tilde O(\cdot)$ to further hide the polylogarithmic factors. 

\newcolumntype{g}{>{\columncolor{LightCyan}}c}

\begin{table*}[t!]
% \begin{adjustbox}{width=\columnwidth,center}
\centering
\tabcolsep=0.02cm
\begin{tabular}{cggggg}
\toprule
\rowcolor{white}
&  &  & Variance & Varying &  
 \\
 \rowcolor{white}
 \multirow{-2}{*}{Model} & \multirow{-2}{*}{Algorithm} & \multirow{-2}{*}{Regret} & -Dependent & Arm Set & \multirow{-2}{*}{Require $B_K$ }
 \\
 \midrule
 \rowcolor{white}
& SW-UCB &  \\
\rowcolor{white}
 &\citep{cheung2018hedging} & \multirow{-2}{*}{$\tilde  O\big(d^{\frac{7}{8}}B_K^{\frac{1}{4}}K^{\frac{3}{4}}\big)$}&\multirow{-2}{*}{No}&\multirow{-2}{*}{Yes}&\multirow{-2}{*}{Yes}\\
  \rowcolor{white}
 \multirow{-4}{*}{Linear Bandit}
 & BOB &\\
 \rowcolor{white}
 &\citep{cheung2018hedging}&\multirow{-2}{*}{ $\tilde  O\Big(d^{\frac{7}{8}}B_K^{\frac{1}{4}}K^{\frac{3}{4}}\Big)$} &\multirow{-2}{*}{No}&\multirow{-2}{*}{Yes}&\multirow{-2}{*}{No} \\
 \rowcolor{white}
 & RestartUCB &  \\
 \rowcolor{white}
 &\citep{zhao2020simple}&\multirow{-2}{*}{ $\tilde  O\Big(d^{\frac{7}{8}}B_K^{\frac{1}{4}}K^{\frac{3}{4}}\Big)$} &\multirow{-2}{*}{No}&\multirow{-2}{*}{Yes}&\multirow{-2}{*}{Yes}\\
  \rowcolor{white}
 & RestartUCB-BOB &  \\
 \rowcolor{white}
 &\citep{zhao2020simple}&\multirow{-2}{*}{ $\tilde  O\Big(d^{\frac{7}{8}}B_K^{\frac{1}{4}}K^{\frac{3}{4}}\Big)$} &\multirow{-2}{*}{No}&\multirow{-2}{*}{Yes}&\multirow{-2}{*}{No}\\
 \rowcolor{white}
 & LB-WeightUCB &  \\
 \rowcolor{white}
 &\citep{wang2023revisiting}&\multirow{-2}{*}{ $\tilde  O\Big(d^{\frac{3}{4}}B_K^{\frac{1}{4}}K^{\frac{3}{4}}\Big)$} &\multirow{-2}{*}{No}&\multirow{-2}{*}{Yes}&\multirow{-2}{*}{Yes}\\
 \rowcolor{white}
  & MASTER + OFUL &  \\
 \rowcolor{white}
 &\citep{wei2021non}&\multirow{-2}{*}{ $\tilde  O\Big(d B_K^{\frac{1}{3}}K^{\frac{2}{3}}\Big)$} &\multirow{-2}{*}{No}&\multirow{-2}{*}{No}&\multirow{-2}{*}{No}\\
 &\small{Restarted-$\algbandit$}&$\tilde O\Big(d^{\frac{7}{8}}(B_KV_K)^{\frac{1}{4}}K^{\frac{1}{2}}$&&&\\
  &\textbf{(Ours)}&$+ d^{\frac{5}{6}}B_K^{\frac{1}{3}}K^{\frac{2}{3}}\Big)$
&\multirow{-2}{*}{Yes}&\multirow{-2}{*}{Yes}&\multirow{-2}{*}{Yes}\\
&Restarted $\text{SAVE}^+$&$\tilde O\Big(d^{\frac{4}{5}}V_K^{\frac{2}{5}}B_K^{\frac{1}{5}}K^{\frac{2}{5}}$&&&\\
  &\textbf{(Ours)}&$+ d^{\frac{2}{3}}B_K^{\frac{1}{3}}K^{\frac{2}{3}}\Big)$
&\multirow{-2}{*}{Yes}&\multirow{-2}{*}{Yes}&\multirow{-2}{*}{Yes}\\
 &Restarted $\text{SAVE}^+\text{-BOB}$&$\tilde O\Big(d^{\frac{4}{5}}V_K^{\frac{2}{5}}B_K^{\frac{1}{5}}K^{\frac{2}{5}}$&&&\\
  &\textbf{(Ours)}&$+ d^{\frac{2}{3}}B_K^{\frac{1}{3}}K^{\frac{2}{3}}+d^{\frac{1}{5}}K^{\frac{7}{10}}\Big)$
&\multirow{-2}{*}{Yes}&\multirow{-2}{*}{Yes}&\multirow{-2}{*}{No}\\
 \midrule
 \rowcolor{white}
MAB 
 % &\citep{wei2016tracking}&\multirow{-2}{*}{ $\tilde  O\big(\left|\mathcal{A}\right|^{\frac{2}{3}}B_K^{\frac{1}{3}}V_K^{\frac{1}{3}}K^{\frac{1}{3}}\big)$} &\multirow{-2}{*}{No}&\multirow{-2}{*}{Yes}&\multirow{-2}{*}{No} \\
  &Rerun-UCB-V&$\tilde O\Big(\left|\mathcal{A}\right|^{\frac{2}{3}}B_K^{\frac{1}{3}}V_K^{\frac{1}{3}}K^{\frac{1}{3}}$&&&\\
  \rowcolor{white}
  &\citep{wei2016tracking}&$+ \left|\mathcal{A}\right|^{\frac{1}{2}}B_K^{\frac{1}{2}}K^{\frac{1}{2}}\Big)$
&\multirow{-2}{*}{Yes}&\multirow{-2}{*}{No}&\multirow{-2}{*}{Yes}\\
\rowcolor{white}
& Lower Bound &  \\
\rowcolor{white}
 &\citep{wei2016tracking} & \multirow{-2}{*}{$\tilde  \Omega\Big(B_K^{\frac{1}{3}}V_K^{\frac{1}{3}}K^{\frac{1}{3}}+B_K^{\frac{1}{2}}K^{\frac{1}{2}}\Big)$}&\multirow{-2}{*}{Yes}&\multirow{-2}{*}{No}&\multirow{-2}{*}{-}\\
\bottomrule
\end{tabular}
\caption{Comparison of non-stationary bandits in terms of regret guarantee. $K$ is the total rounds, $d$ is the problem dimension for linear bandits, $B_K$ is the \emph{total variation budget} defined in Eq.(\ref{eq:variation_budget}) (for the MAB setting, $B_K=\sum_{k=1}^K \|\mu_k-\mu_{k+1}\|_{\infty}$, where $\mu_k$ is the mean of the reward distribution at round $k$), $V_K$ is the \emph{total variance} defined in Eq.(\ref{total variance}), $\left|\mathcal{A}\right|$ is the number of arms for Multi-Armed bandits. }\label{table:11}
% \end{adjustbox}
\end{table*}

\section{Related Work}
Our work is closely related to two lines of research: non-stationary (linear) bandits and linear bandits with heteroscedastic noises. 
\subsection{Non-stationary (Linear) Bandits}
There have been a series of works about non-stationary bandits \citep{auer2002nonstochastic,10.1007/978-3-642-24412-4_16,besbes2014stochastic,wei2016tracking,cheung2019learning,russac2019weighted,pmlr-v99-auer19a,pmlr-v99-chen19b,russac2020algorithms,zhao2020simple,kim2020randomized,wei2021non,russac2021self,chen2021combinatorial,deng2022weighted,suk2022tracking,liu2023definition,abbasi2023new,clerici2023linear}. In non-stationary linear bandits, the unknown feature vector $\btheta_k$ can be dynamically and adversarially adjusted, with the total change upper bounded by the \emph{total variation budget} $B_K$ over $K$ rounds, \emph{i.e.}, $\sum_{k=1}^{K-1} \|\btheta_{k+1}-\btheta_k\|_2\leq B_K$. To tackle this problem, some works proposed forgetting strategies such as sliding window, restart, and weighted regression \citep{cheung2019learning,russac2019weighted,zhao2020simple}.  \citet{kim2020randomized} also introduced the randomized exploration with weighting strategy. The regret upper bounds in these works are all of $\Tilde{O}(B_K^{\frac{1}{4}}K^{\frac{3}{4}})$. A recent work by  \cite{wei2021non} proposed the MASTER-OFUL algorithm based on a black-box approach, which can achieve a regret bound of $\Tilde{O}(B_K^{\frac{1}{3}}K^{\frac{2}{3}})$ in the case where the arm set is fixed over $K$ rounds. To the best of our knowledge, none of the existing works consider how to utilize the variance information to improve the regret bound in the case with time-dependent variances. The only exception of utilizing the variance information in the non-stationary bandit setting is \citet{wei2016tracking}, which proposed the Rerun-UCB-V algorithm for the non-stationary MAB setting with a regret dependent on the action set size $|\cA|$. To compare with, the regret upper bounds of our algorithms are independent of the action set size, thus our algorithms are more efficient for the case where the number of actions is large.  

% Randomized exploration with weighting strategy has also been introduced in~\citet{kim2020randomized}. The aforementioned works provide a regret of $\tilde{\calO}(d^{2/3} \Delta^{1/2} K^{2/3})$ which is optimal as it matches the established lower bound $\Omega (d^{2 /3} \Delta^{1 /3} K^{2 / 3})$ in~\citep{cheung2019learning} up to $\log(K)$ factors. 
\subsection{Linear Bandits with Heteroscedastic Noises}
Some recent works study the heteroscedastic linear bandit problem, where the noise distribution is assumed to vary over time. \citet{kirschner2018information} first proposed the linear bandit model with heteroscedastic noise. In this model, the noise at round $k \in [K]$ is assumed to be $\sigma_k$-sub-Gaussian. Some follow-up works relaxed the $\sigma_k$-sub-Gaussian assumption by assuming the noise at the $k$-th round to be of variance $\sigma_k^2$ \citep{zhou2021nearly, zhang2021improved, kim2022improved, zhou2022computationally, dai2022variance,zhao2023variance}. Specifically, \citet{zhou2021nearly} and \citet{zhou2022computationally} considered the case where $\sigma_k$ is observed by the learner after the $k$-th round. \cite{zhang2021improved} and \cite{kim2022improved} proposed statistically efficient but computationally inefficient algorithms for the unknown-variance case. A recent work by \cite{zhao2023variance} proposed an algorithm that achieves both statistical and computational efficiency in the unknown-variance setting. \citet{dai2022variance} also considered a  specific heteroscedastic linear bandit problem where the linear model is sparse.

\section{Problem Setting}
    We consider a heteroscedastic variant of the classic non-stationary linear contextual bandit problem. Let $K$ be the total number of rounds. At each round $k\in[K]$, the interaction between the learner and the environment is as follows: 
   (1) the environment generates an arbitrary arm set $\cD_k \subseteq \RR^d$ where each element represents a feasible arm for the learner to choose, and also generates an \emph{unknown} feature vector $\btheta_k$;
        (2) the leaner observes $\cD_k$ and selects $\ab_k \in \cD_k$;
        (3) the environment generates the stochastic noise $\epsilon_k$ at round $k$ and reveals the stochastic reward $r_k = \la \btheta_k, \ab_k \ra + \epsilon_k$ to the leaner. 
        % (1) the environment generates an arbitrary arm set $\cD_k \subseteq \RR^d$ where each element represents a feasible arm for the learner to choose, and also generates an \emph{unknown} feature vector $\btheta_k$; (2) the leaner observes $\cD_k$ and selects $\ab_k \in \cD_k$; and (3) the environment generates the stochastic noise $\epsilon_k$ at round $k$ and reveal the stochastic reward $r_k = \la \btheta_k, \ab_k \ra + \epsilon_k$ to the leaner. 
        We assume the $\ell_2$ norm of the feasible actions is upper bounded by $A$, \emph{i.e.}, for all $k \in [K]$, $\ab \in \cD_k$: $\|\ab\|_2 \le A$. 
        
        Following \citet{zhou2021nearly,zhao2023variance}, we assume the following condition on the random noise $\epsilon_k$ at each round $k$: 
    \begin{align} 
        \PP\left(|\epsilon_k| \le R\right) = 1&,\quad \EE[\epsilon_k | \ab_{1:k}, \epsilon_{1:k - 1}] = 0, 
        \quad\EE[\epsilon_k^2 | \ab_{1:k}, \epsilon_{1:k - 1}] = \sigma_k^2. \label{eq:noise:condition}
    \end{align}
For the case with known variance, we assume that at each round $k$, the variance $\sigma_k$ is also revealed to the learner together with the reward $r_k$, whereas in the unknown variance case, the $\sigma_k$ can not be observed.

    Following \cite{cheung2018hedging,cheung2019learning,russac2019weighted,zhao2020simple}, we assume the sum of $\ell_2$ differences of consecutive $\theta_k$'s is upper bounded by the \emph{total variation budget} $B_K$, \emph{i.e.}, \begin{align}\label{eq:variation_budget}\sum_{k=1}^{K-1}\|\btheta_{k+1}-\btheta_k\|_2\leq B_K\,,\end{align}
where the $\btheta_k$'s can be adversarially chosen by an oblivious adversary. We also assume that the \emph{total variance} is upper bounded by  $V_K$, which is
\begin{align}
    \sum_{k=1}^K \sigma_k^2\leq V_K.\label{total variance}
\end{align}

The goal of the agent is to minimize the \emph{dynamic regret} defined as follows: 
\begin{align}
    \text{Regret}(K)&=\sum_{k\in[K]} \big(\la\ab_k^*,\btheta_k\ra-\la\ab_k,\btheta_k\ra\big)\notag\,,
\end{align}
where $\ab_k^*=\argmax_{\ab\in\cD_k}\la\ab,\btheta_k\ra$ is the optimal arm at round $k$ which gives the highest expected reward.

\section{Non-stationary Linear Contextual Bandit with Known Variance}
\begin{algorithm}[t]	\caption{Restarted-$\algbandit$}\label{algorithm:reweightbandit}
	\begin{algorithmic}[1]
	\REQUIRE Regularization parameter $\lambda>0$; 
	$\pnorm$, 
	an upper bound on the $\ell_2$-norm of $\btheta_k$ for all $k\in[K]$; confidence radius $\hat\beta_k$, variance parameters $\alpha, \gamma$; restart window size $w$.
	\STATE $\hat\bSigma_1 \leftarrow \lambda \Ib$, $\hat\bbb_1 \leftarrow \zero$, $\hat\btheta_1 \leftarrow \zero$, $\hat\beta_1 = \sqrt{\lambda}\pnorm$
	\FOR{$k=1,\ldots, K$}
        \IF{$k\% w == 0$}
\STATE $\hat\bSigma_k \leftarrow \lambda \Ib$, $\hat\bbb_k \leftarrow \zero$, $\hat\btheta_k \leftarrow \zero$, $\hat\beta_k = \sqrt{\lambda}\pnorm$
        \ENDIF
	\STATE Observe $\cD_k$ and choose $\ab_k\leftarrow\argmax_{\ab\in\cD_k} \la\ab,\btheta_k\ra+\hat\beta_k\|\ab_k\|_{\hat\bSigma_k^{-1}}$ 
	\STATE Observe $(r_k,\sigma_k)$, set $\bar\sigma_k$ as 
	\begin{align}
	    &\bar\sigma_k \leftarrow \max\{\sigma_k, \alpha, \gamma\|\ab_k\|_{\hat\bSigma_k^{-1}}^{1/2}\}\label{def:banditvar}
	\end{align}
	\STATE $\hat\bSigma_{k+1} \leftarrow \hat\bSigma_k + \ab_k\ab_k^\top/\bar\sigma_k^2$, $\hat\bbb_{k+1} \leftarrow \hat\bbb_k + r_k\ab_k/\bar\sigma_k^2$, $\hat\btheta_{k+1}\leftarrow \hat\bSigma_{k+1}^{-1}\hat\bbb_{k+1}$\label{alg:reweightbandit}
	\ENDFOR
	\end{algorithmic}
\end{algorithm}
In this section, we introduce our Algorithm \ref{algorithm:reweightbandit} under the setting where the variance $\sigma_k^2$ at $k$-th iteration is known to the agent in prior. We start from WeightedOFUL$^+$ \citep{zhou2022computationally}, an \emph{weighted ridge regression}-based algorithm for heteroscedastic linear bandits under the stationary reward assumption. For our non-stationary linear bandit setting where $\btheta_k$ is changing over the round $k$, WeightedOFUL$^+$ aims to build an $\hat\btheta_k$ which estimates the feature vector $\btheta_k$ by using the solution to the following regression problem:
\begin{align}
    \hat\btheta_k\leftarrow \arg\min_{\btheta}\sum_{t=1}^{k-1}\bar\sigma_t^{-2}(\la \btheta, \ab_t\ra - r_t)^2 + \lambda \|\btheta\|_2^2,\label{hhh}
\end{align}
where the weight is defined as in \eqref{def:banditvar}. After obtaining $\hat\btheta_k$, WeightedOFUL$^+$ chooses arm $\ab_k$ by maximizing the upper confidence bound (UCB) of $\la \ab, \hat\btheta\ra$, with an exploration bonus $\hat\beta_k\|\ab_k\|_{\hat\bSigma_k^{-1}}$, where $\hat\bSigma_k$ is the covariance matrix over $\ab_k$. 
The weight $\bar\sigma_k^2$ is introduced to balance the different past examples based on their reward variance $\sigma_k^2$, and such a strategy has been proved as a state-of-the-art algorithm for the stationary heteroscedastic linear bandits \citep{zhou2022computationally}. However, the non-stationary nature of our setting prevents us from directly using $\hat\btheta_k$ defined in \eqref{hhh} as an estimate to $\btheta$. Therefore, inspired by the \emph{restarting} strategy which has been adopted by previous algorithms for non-stationary linear bandits \citep{zhao2020simple}, we propose Restarted-WeightedOFUL$^+$, which periodically restarts itself and runs WeightedOFUL$^+$ as its submodule. The restart window size is set as $w$, which is used to balance the nonstationarity and the total regret and will be fine-tuned in the next steps. Combined with the restart window size $w$, we set $\{\hat\beta_k\}_{k\geq 1}$ to
\begin{small}\begin{align}
    \hat\beta_k& = 12\sqrt{d\log(1+\frac{(k\%w)A^2}{\alpha^2d\lambda})\log(32(\log(\frac{\gamma^2}{\alpha}+1)\frac{(k\%w)^2}{\delta})}  + 30\log(32(\log(\frac{\gamma^2}{\alpha})+1)\frac{(k\%w)^2}{\delta})\frac{R}{\gamma^2}+ \sqrt{\lambda}\pnorm.
\label{eq:defbanditbeta}
\end{align}\end{small}

We now propose the theoretical guarantee for Algorithm \ref{alg:reweightbandit}. The following key lemma shows how nonstationarity affects our estimation of the reward of each arm.

\begin{lemma}\label{lemma:key}
Let $0<\delta<1$. Then with probability at least $1-\delta$, for any action $\ab \in \RR^d$, we have
    \begin{align}
    |\ab^\top(\hat\btheta_k-\btheta_k)|&\leq \underbrace{\frac{A^2}{\alpha}\sqrt{\frac{dw}{\lambda}}\sum_{t=w\cdot \lfloor k/ w\rfloor+1}^{k-1} \|\btheta_t-\btheta_{t+1}\|_2}_{\text{Drifting term}} +\underbrace{\hat\beta_k\|\ab\|_{\hat\bSigma_k^{-1}}}_{\text{Stochastic term}}.\notag
\end{align}
\end{lemma}
\begin{proof}
    See Appendix \ref{app:keylemma} for the full proof. 
\end{proof}
Lemma \ref{lemma:key} suggests that under the non-stationary setting, the difference between the true expected reward and our estimated reward will be upper bounded by two separate terms. The first drifting term charcterizes the error caused by the non-stationary environment, and the second stochastic term charcterizes the error caused by the estimation of the stochastic environment. Note that similar bound has also been discovered in \citet{touati2020efficient}. We want to emphasize that our bound differs from existing ones in 1) an additional variance parameter $\alpha$ in the drifting term, and 2) a weighted convariance matrix $\hat\bSigma$ rather than a vanilla convariance matrix.

Next we present our main theorem.
\begin{theorem}\label{thm: regret for algo1 final}
Let $0<\delta<1$. Suppose that for all $k \geq 1$ and all $\ab \in \cD_k$, $\la \ab, \btheta_k\ra \in [-1, 1]$, $\|\btheta^*\|_2 \leq \pnorm$, $\|\ab\|_2 \leq A$.
 With probability at least $1-\delta$, the regret of Restarted-$\algbandit$ is bounded by
\begin{align}
    \text{Regret}(K)
    & \leq \frac{2A^2B_Kw^{\frac{3}{2}}}{\alpha}\sqrt{\frac{d}{\lambda}}+ 4\hat\beta\sqrt{V_K + K\alpha^2}\sqrt{\frac{Kd\iota}{w}}+ \frac{4d\iota K\hat\beta\gamma^2}{w}+\frac{4d\iota K}{w},\label{eq:cororegret}
\end{align}
where $\iota = \log(1+\frac{wA^2}{d\lambda\alpha^2})$, and $\hat\beta= \tilde O(\sqrt{d} + R/\gamma^2 + \sqrt{\lambda}\pnorm)$. Specifically, by treating $A,\lambda, B, R$ as constants and setting $\gamma^2 = R/\sqrt{d}$, we have
\begin{align}
     \text{Regret}(K) &= \tilde O(B_Kw^{3/2}d^{1/2}\alpha^{-1}+ dK\alpha/\sqrt{w}   + d\sqrt{K V_K/w} + dK/w).\label{eq:ttt}
\end{align}
\end{theorem}
\begin{proof}
    See Appendix \ref{app:firstalg}.
\end{proof}
\begin{remark}\label{rmk:1}
    For the stationary linear bandit case where $B_K = 0$, we can set the restart window size $w = K$ and the variance parameter $\alpha = 1/\sqrt{K}$, then we obtain an $\tilde O(d\sqrt{V_K} + d)$ regret for Algorithm \ref{alg:reweightbandit}, which is identical to the one in \citet{zhou2022computationally}. 
\end{remark}

Next, we aim to select parameters $\alpha$ and $w$ in order to optimize \eqref{eq:ttt}. 
\begin{corollary}
    Assume that $B_K, V_K \in [\Omega(1), O(K)]$. Then by selecting
    \begin{align}
        &w=d^{1/4}\sqrt{V_K/B_K}, &dV_K^6\geq K^4B_K^2,\notag\\
        &w=d^{1/6}(K/B_K)^{1/3}&\text{otherwise}.\notag
    \end{align}
    and $\alpha = d^{-1/4}B_K^{1/2}wK^{-1/2}$, the final regret is in the order
    \begin{align}
        \text{Regret}(K) &=\tilde O(d^{7/8}(B_KV_K)^{1/4}\sqrt{K} + d^{5/6}B_K^{1/3}K^{2/3}). 
    \end{align}\label{corollary for algo1 regret}
\end{corollary}

\begin{remark}
We compare the regret of Algo.\ref{alg:reweightbandit} in Corollary \ref{corollary for algo1 regret} with previous results in the following special cases.
\begin{itemize}
    \item In the worst case where $V_K=O(K)$, our result becomes $\Tilde{O}(d^{7/8}B_K^{1/4}K^{3/4})$, matching the state-of-the-art results for restarting and sliding window strategies \cite{cheung2018hedging,zhao2020simple}.
    \item In the case where the \emph{total variance} is small, \emph{i.e.}, $V_K=\Tilde{O}(1)$,  assuming that $K^4> d$, our result becomes $\Tilde{O}(d^{5/6}B_K^{1/3}K^{2/3})$, better than all the previous results \cite{cheung2018hedging,zhao2020simple,wang2023revisiting,wei2021non}.
\end{itemize}
\end{remark}

\begin{remark}
    \citet{wei2016tracking} has studied non-stationary MAB with dynamic variance. With the knowledge of $V_K$ and $B_K$, \citet{wei2016tracking} proposed a restart-based Rerun-UCB-V algorithm with a $\tilde O(\left|\mathcal{A}\right|^{\frac{2}{3}}B_K^{\frac{1}{3}}V_K^{\frac{1}{3}}K^{\frac{1}{3}}+ \left|\mathcal{A}\right|^{\frac{1}{2}}B_K^{\frac{1}{2}}K^{\frac{1}{2}})$ regret, where $\cA$ is the action set. Reduced to the MAB setting, our Restarted-$\algbandit$ achieves an $\tilde O(|\cA|^{7/8}(B_KV_K)^{1/4}\sqrt{K} + |\cA|^{5/6}B_K^{1/3}K^{2/3})$ regret, which is worse than \citet{wei2016tracking}. We claim that this is due to the generality of the linear bandits, which brings us a looser bound to the drifting term in Lemma \ref{lemma:key}. When restricting to the MAB setting, our drifting term enjoys a tighter bound, which could further tighten our final regret. To develop an algorithm achieving the same regret as \citet{wei2016tracking} is beyond the scope of this work.
\end{remark}

\begin{remark}
    \citet{wei2016tracking} has established a lower bound $\tilde  \Omega(B_K^{\frac{1}{3}}V_K^{\frac{1}{3}}K^{\frac{1}{3}}+B_K^{\frac{1}{2}}K^{\frac{1}{2}})$ for MAB with total variance $V_K$ and total variation budget $B_K$. There still exist gaps between our regret and their lower bound regarding the dependence of $K, V_K, B_K$, and we leave to fix the gaps as future work.
\end{remark}

\section{Non-stationary Linear Contextual Bandit with Unknown Variance and Total Variation Budget}
By Theorem \ref{thm: regret for algo1 final}, we know that Algorithm \ref{alg:reweightbandit} is able to utilize the total variance $V_K$ and obtain a better regret result compared with existing algorithms which do not utilize $V_K$. However, the success of Algorithm \ref{alg:reweightbandit} depends on the knowledge of the per-round variance $\sigma_k$, and it also depends on a good selection of restart window size $w$, whose optimal selection depends on both $V_K$ and $B_K$. In this section, we aim to relax these two requirements with still better regret results.

\subsection{Unknown Per-round Variance, Known $V_K$ and $B_K$}
\begin{algorithm*}[t!] 
    \caption{$\text{Restarted SAVE}^+$} \label{alg:1}
    \begin{algorithmic}[1]
        \REQUIRE $\alpha > 0$; the upper bound on the $\ell_2$-norm of $\ab$ in $\cD_k (k\ge 1)$, i.e., $A$; the upper bound on the $\ell_2$-norm of $\btheta_k$ $(k\ge 1)$, i.e., $\pnorm$; restart window size $w$.
    \STATE Initialize $L \leftarrow \lceil \log_2 (1 / \alpha) \rceil$. 
    \STATE Initialize the estimators for all layers: $\hat\bSigma_{1, \ell} \leftarrow 2^{-2\ell} \cdot \Ib$, $\hat\bbb_{1, \ell} \leftarrow \zero$, $\hat\btheta_{1, \ell} \leftarrow \zero$, $\hat \beta_{1, \ell} \leftarrow 2^{-\ell + 1}$, $\hat\Psi_{1,\ell}\leftarrow \emptyset$ for all $\ell \in [L]$. \alglinelabel{alg1:line:2}
    \FOR{$k=1,\ldots, K$}
    \IF{$k\%w==0$}
\STATE Set $\hat\bSigma_{k, \ell} \leftarrow 2^{-2\ell} \cdot \Ib$, $\hat\bbb_{k, \ell} \leftarrow \zero$, $\hat\btheta_{k, \ell} \leftarrow \zero$, $\hat \beta_{1, \ell} \leftarrow 2^{-\ell + 1}$, $\hat\Psi_{k,\ell}\leftarrow \emptyset$ for all $\ell \in [L]$.\alglinelabel{alg:restart}
\ENDIF
\STATE Observe $\cD_k$, choose $\ab_k \leftarrow \argmax_{\ab \in \cD_k} \min_{\ell \in [L]}\la \ab, \hat{\btheta}_{k, \ell}\ra + \hat\beta_{k, \ell} \|\ab\|_{\hat\bSigma_{k, \ell}^{-1}}$ and observe $r_k$. \alglinelabel{line:selection} 
\STATE Set $\ell_k\leftarrow L+1$
\STATE Let $\cL_k \leftarrow \{\ell \in [L]: \|\ab_k\|_{\hat\bSigma_{k,\ell}^{-1}} \geq 2^{-\ell}\}$, set $\ell_k\leftarrow \min(\cL_k)$ if $\cL_k \neq \emptyset$\alglinelabel{alg1:line:9}
\STATE $\hat\Psi_{k,\ell_k} \leftarrow \hat\Psi_{k,\ell_k}\cup \{k\}$
\IF{$\cL_k \neq \emptyset$}
\STATE Set $w_k \leftarrow \frac{2^{-\ell_k}}{{\|\ab_k\|_{\hat\bSigma_{k, \ell_k}^{-1}}}}$ and update \alglinelabel{alg1:line:12} 
\begin{align} \hat\bSigma_{k + 1, \ell_k} \leftarrow \hat\bSigma_{k, \ell_k} + w_k^2 \ab_k \ab_k^{\top}, \hat{\bbb}_{k + 1, \ell} \leftarrow \hat\bbb_{k, \ell_k} + w_k^2 \cdot r_k \ab_k, \hat\btheta_{k + 1, \ell_k} \leftarrow \hat\bSigma_{k + 1, \ell_k}^{-1} \hat{\bbb}_{k + 1, \ell_k}. \notag \end{align}
\STATE Compute the adaptive confidence radius $\hat\beta_{k+1, l}$for the next round according to \eqref{eq:def:beta}. \label{alg1:line:18}
\ENDIF
     \STATE For $\ell\neq \ell_k$ let $\hat\bSigma_{k + 1, \ell} \leftarrow \hat\bSigma_{k, \ell}, \hat\bbb_{k + 1, \ell} \leftarrow \hat\bbb_{k, \ell}, \hat\btheta_{k + 1, \ell} \leftarrow \hat\btheta_{k, \ell}, \hat\beta_{k + 1, \ell} \leftarrow \hat\beta_{k, \ell}.$
    \ENDFOR
    \end{algorithmic}
\end{algorithm*}

We first aim to relax the requirement that each $\sigma_k^2$ is known to the agent at the beginning of $k$-th round. We follow the SAVE algorithm \citep{zhao2023variance} which introduces a multi-layer structure \citep{chu2011contextual, he2021uniform} to deal with unknown $\sigma_k^2$. In detail, SAVE maintains multiple estimates to the current feature vector $\theta_k$, which we denote them as $\hat\btheta_{k,1},...,\hat\btheta_{k,L}$ in line \ref{alg1:line:2}. Each $\hat\btheta_{k,\ell}$ is calculated based on a subset $\hat\Psi_{k, \ell} \subseteq [k-1]$ of samples $\{(\ab_t, r_t)\}$. The rule that whether to add the current $k$ to some $\hat\Psi_{k, \ell}$ is based on the uncertainty of $\ab_k$ with the sample set $\{(\ab_t, r_t)\}_{t \in \hat\Psi_{k, \ell}}$. As long as $\ab_k$ is too uncertain w.r.t. some level $\ell_k$ (line \ref{alg1:line:9}), we add $k$ to $\hat\Psi_{k, \ell}$ and update the estimate $\hat\btheta_{k,\ell_k}$ accordingly (line \ref{alg1:line:12}). Each $\hat\btheta_{k,\ell_k}$ is calculated as the solution of a weighted regression problem, where the weight $w_k$ is selected as the inverse of the uncertainty of the arm $\ab_k$ w.r.t. the samples in the $\ell$-th layer. Maintaining $L$ different $\hat\btheta_{k,\ell}, \ell \in [L]$, Algorithm \ref{alg:1} then calculates $L$ number of UCB for each arm $\ab$ w.r.t. $L$ different $\hat\btheta_{k,\ell}$, and selects the arm which maximizes the minimization of $L$ UCBs (line \ref{line:selection}). It has been shown in \citet{zhao2023variance} that such a multilayer structure is able to utilize the $V_K$ information without knowing the per-round variance $\sigma_k^2$. Similar to Algorithm \ref{alg:reweightbandit}, in order to deal with the nonstationarity issue, we introduce a restarting scheme that Algorithm \ref{alg:1} restarts itself by a restart window size $w$ (line \ref{alg:restart}).

Next we show the theoretical guarantee of Algorithm \ref{alg:1}. We call the restart time rounds \emph{grids} and denote them by $g_1, g_2, \ldots g_{\lceil \frac{K}{w}\rceil-1}$, where $g_i\%w=0$ for all $i\in[\lceil \frac{K}{w}\rceil-1]$. Let $i_k$ be the grid index of time round $k$, \emph{i.e.}, $g_{i_k}\leq k<g_{i_k+1}$. We denote $\hat\Psi_{k, \ell}:=\{t: t\in[g_{i_k},k-1], \ell_t=\ell\}$.
We define the confidence radius $\hat\beta_{k,\ell}$ at round $k$ and layer %\todoq{layer, or level, need to be consistent} 
$\ell$ as
\begin{small}
    \begin{align} 
    \hat \beta_{k, \ell} &:= 16 \cdot 2^{-\ell} \sqrt{\left(8\hat\Var_{k, \ell}  + 6R^2 \log(\frac{4(w + 1)^2 L}{\delta}) + 2^{-2\ell + 4}\right)}  \times\sqrt{\log(\frac{4w^2 L }{\delta} )}+ 6 \cdot 2^{-\ell} R \log(\frac{4w^2 L}{\delta}) + 2^{-\ell}B\label{eq:def:beta}, 
\end{align}
\end{small}
where 
\begin{small}
    $\hat\Var_{k, \ell} := \begin{cases}\sum_{i \in \hat\Psi_{k, \ell}} w_i^2 \big(r_i - \la \hat\btheta_{k, \ell}, \ab_i \ra \big)^2, &\text{If} \quad 2^\ell \ge 64 \sqrt{\log\left(\frac{4(w + 1)^2 L}{\delta}\right)} \\
R^2 \left|\hat\Psi_{k, \ell}\right|,  &\text{otherwise}.
\end{cases} $
\end{small}

Note that our selection of the confidence radius $\hat\beta_{k, \ell}$ only depends on $\hat\Var_{k, \ell}$, which serves as an estimate of the total variance of samples at $\ell$-th layer without knowing $\sigma_k^2$. 

We build the theoretical guarantee of Algorithm \ref{alg:1} as follows.

 \begin{theorem} \label{thm:regret1}
 Let $0<\delta<1$. Suppose that for all $k \geq 1$ and all $\ab \in \cD_k$, $\la \ab, \btheta_k\ra \in [-1, 1]$, $\|\btheta^*\|_2 \leq \pnorm$, $\|\ab\|_2 \leq A$. If $\{\beta_{k, \ell}\}_{k\ge 1, \ell \in [L]}$ is defined in \eqref{eq:def:beta},
 % and $\alpha = 1 / (R \cdot K^{3 / 2})$,
 then the cumulative regret of Algorithm~\ref{alg:1} is bounded as follows with probability at least $1 - 3\delta$: %\todoq{since you're using big O notation, you can use $=$ instead of $\leq$. This applies to other theorems/corolaries a well (fixed)}
        \begin{align} 
            \text{Regret}(K)&= \tilde{O}\bigg(\frac{A^2\sqrt{d}w^{\frac{3}{2}}B_K}{\alpha}+\big(w\alpha^2+d\big)\cdot\sqrt{\frac{K}{w}V_K}+\big(1+R\big)\cdot\big(K\alpha^2+\frac{Kd}{w}\big)\bigg)
        \end{align}
        Specifically, regarding $A, R$ as constants, we have
        \begin{align}
            \text{Regret}(K)&=\tilde O(\sqrt{d}w^{1.5}B_K/\alpha + \alpha^2(K+\sqrt{wKV_K}) + d\sqrt{KV_K/w} + dK/w).\notag
        \end{align}
    \end{theorem}
    \begin{proof}
    See Appendix \ref{app:secalg} for the full proof.
\end{proof}
    \begin{remark}
    Like Remark \ref{rmk:1}, we consider the case where $B_K = 0$. We set $w = K$ and $\alpha^2 = 1/K\sqrt{V_K}$, then we obtain a regret $\tilde O(d\sqrt{V_K} +d)$, which matches the regret of the SAVE algorithm in \citet{zhao2023variance}. 
\end{remark}
% We have the following corollary to optimize the regret if we have no knowledge of $B_K$ and $V_K$. 
% \begin{corollary}
% By selecting 
% \begin{align}
%     w=(dK)^{1/3}, \alpha=d^{1/6}\sqrt{w}/K^{1/3}=d^{1/3}/K^{1/6},
% \end{align}
% we have
% \begin{align}
% &\sqrt{d}\sqrt{dK}B_KK^{2/3}/d^{1/3} + d^{2/3}K^{2/3} + d^{2/3}/K^{1/3}d^{1/6}K^{1/6}\sqrt{KV_K} + d^{5/6}K^{1/3}\sqrt{V_K} + d^{2/3}K^{2/3}\notag \\
% &=d^{2/3}K^{2/3}B_K + d^{5/6}K^{1/3}\sqrt{V_K}
% \end{align}
% \end{corollary}

\begin{corollary}
    Assume that $B_K, V_K \in [\Omega(1), O(K)]$, then by selecting
\begin{align}
        &w=d^{1/3}(K/B_K)^{1/3}, &K^2\geq V_K^3d/B_K,\notag\\
        &w=d^{2/5}(KV_K)^{1/5}/B_K^{2/5}&\text{otherwise}.\notag
    \end{align}
    and $\alpha = d^{1/6}\sqrt{w}B_K^{1/3}/(K^{1/3} + (V_KKw)^{1/6})$, we have
    \begin{align}
        \text{Regret}(K) = \tilde O(d^{4/5}V_K^{2/5}B_K^{1/5}K^{2/5} + d^{2/3}B_K^{1/3}K^{2/3}).
    \end{align}\label{corollary w alpha opt}
\end{corollary}
\begin{remark}
We discuss the regret of Algo.\ref{alg:1} in Corollary \ref{corollary w alpha opt} in the following special cases. In the case where the \emph{total variance} is small, \emph{i.e.}, $V_K=\Tilde{O}(1)$,  assuming that $K^2> d$, our result becomes $\Tilde{O}(d^{2/3}B_K^{1/3}K^{2/3})$, better than all the previous results \cite{cheung2018hedging,zhao2020simple,wang2023revisiting,wei2021non}. In the worst case where $V_K=O(K)$, our result becomes $\Tilde{O}(d^{4/5}B_K^{1/5}K^{4/5})$.
\end{remark}

% \zhiyong{If we use BOB, we can get     \begin{align}
%         \text{Regret}(K) = \tilde O(d^{4/5}V_K^{2/5}B_K^{1/5}K^{2/5} + d^{2/3}B_K^{1/3}K^{2/3}+\sqrt{wK})
%     \end{align}}

\subsection{Unknown Per-round Variance, Unknown $V_K$ and $B_K$}
In Corollary \ref{corollary w alpha opt}, we need to know the \emph{total variance} $V_K$ and \emph{total variation budget} $B_K$ to select the optimal $w$ and $\alpha$. To deal with the more general case where $V_K$ and $B_K$ are unknown, we can employ the \emph{Bandits-over-Bandits} (BOB) mechanism proposed in \cite{cheung2019learning}. We name the Restarted $\text{SAVE}^+$ algorithm with BOB mechanism as “Restarted $\text{SAVE}^+$-BOB”. Due to the space limit, we put the detailed algorithm of Restarted $\text{SAVE}^+$-BOB to Algorithm \ref{alg:bob algo} in Appendix \ref{app:bob}, and we present the main idea of Restarted $\text{SAVE}^+$-BOB as follows.

We divide the $K$ rounds into $\lceil\frac{K}{H}\rceil$ blocks, with each block having $H$ rounds (except the last one may have less than $H$ rounds).
Within each block $i$, we use a fixed $(\alpha_i, w_i)$ pair to run the Restarted $\text{SAVE}^+$ algorithm. To adaptively learn the optimal $(\alpha, w)$ pair without the knowledge of $V_K$ and $B_K$, we employ an 
adversarial bandit algorithm (Exp3 in \cite{auer2002nonstochastic}) as the meta-learner to select $\alpha_i, w_i$ over time for $i\in\lceil\frac{K}{H}\rceil$ blocks. Specifically, in each block, the meta learner selects a $(\alpha, w)$ pair from the candidate pool to feed to Restarted $\text{SAVE}^+$, and the cumulative reward received by Restarted $\text{SAVE}^+$ within the block is fed to the meta-learner as the reward feedback to select a better pair for the next block.

We set the block length $H$ to be $\lceil d^{\frac{2}{5}}K^{\frac{2}{5}}\rceil$, and set the candidate pool of $(\alpha, w)$ pairs for the Exp3 algorithm as:
\begin{align}
    \mathcal{P}=\{(w,\alpha): w\in\mathcal{W}, \alpha\in\mathcal{J}\}\,,\label{bob pool}
\end{align}
where
\begin{align}
\mathcal{W}&=\{w_i=d^{\frac{1}{3}}2^{i-1}|i\in\lceil\frac{1}{3}\log_2 K\rceil+1\}\cup\{w_i=d^{\frac{2}{5}}2^{i-1}|i\in\lceil\frac{2}{5}\log_2 K\rceil+1\}\,,\label{w set}
\end{align}
and 
\begin{align}
    \mathcal{J}&=\{\alpha_i=d^{\frac{1}{3}}2^{-i+1}|i\in\lceil\frac{1}{3}\log_2 K\rceil+1\}\cup\{\alpha_i=d^{\frac{11}{30}}2^{-i+1}|i\in\lceil\frac{11}{30}\log_2 K\rceil+1\}\,.\label{alpha set}
\end{align}

% We refer the readers to Section 7 in \cite{cheung2019learning} for more algorithm details of the BOB framework.

\begin{figure*}[h]
    \centering
     \subfigure[$B_K=1$]{
    \includegraphics[scale=0.24]{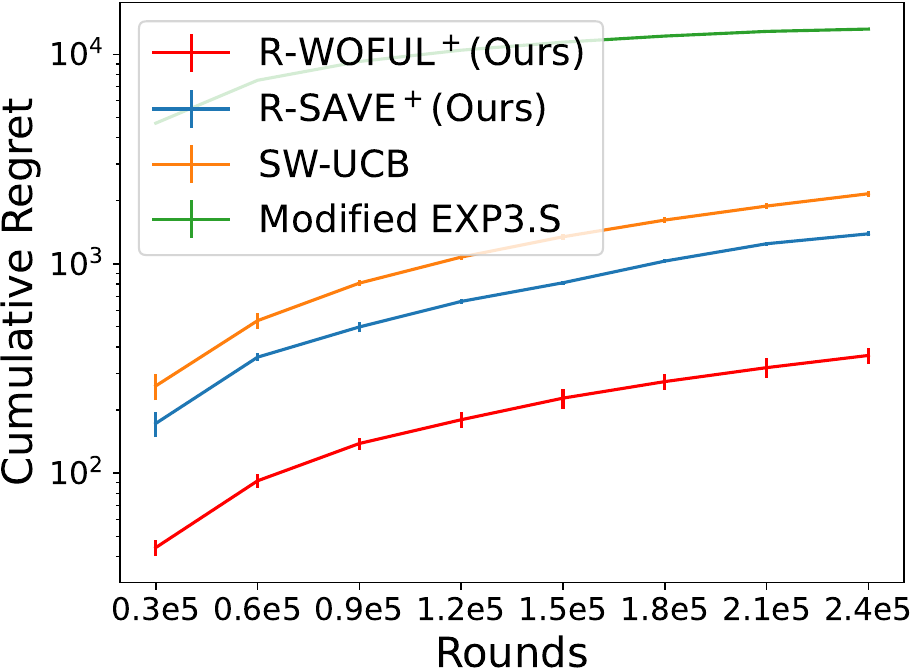}
    }
    \hfill
     \subfigure[$B_K=10$]{
    \includegraphics[scale=0.24]{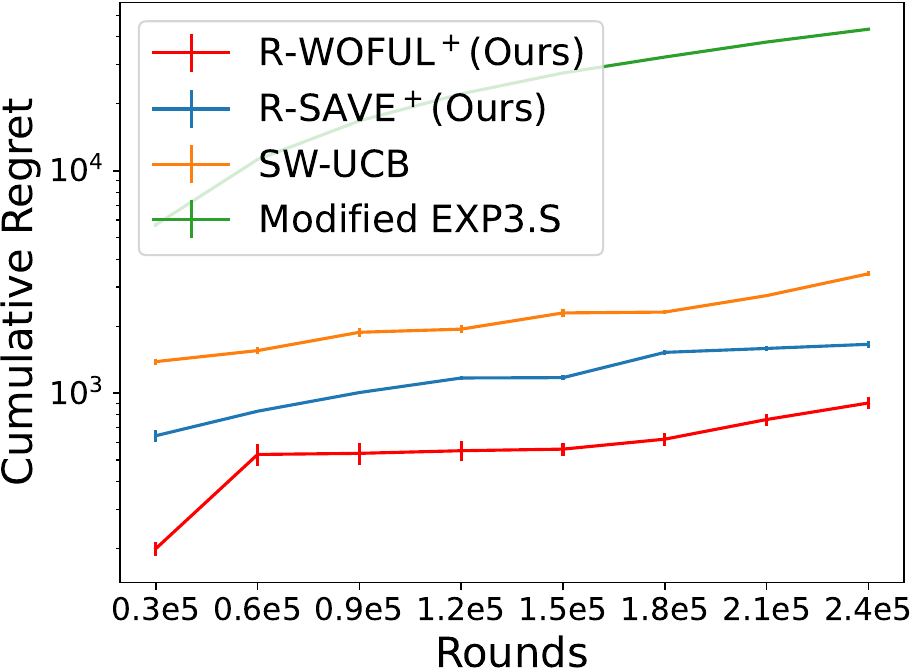}
    }
    \hfill
    \subfigure[$B_K=20$]{
    \includegraphics[scale=0.24]
    {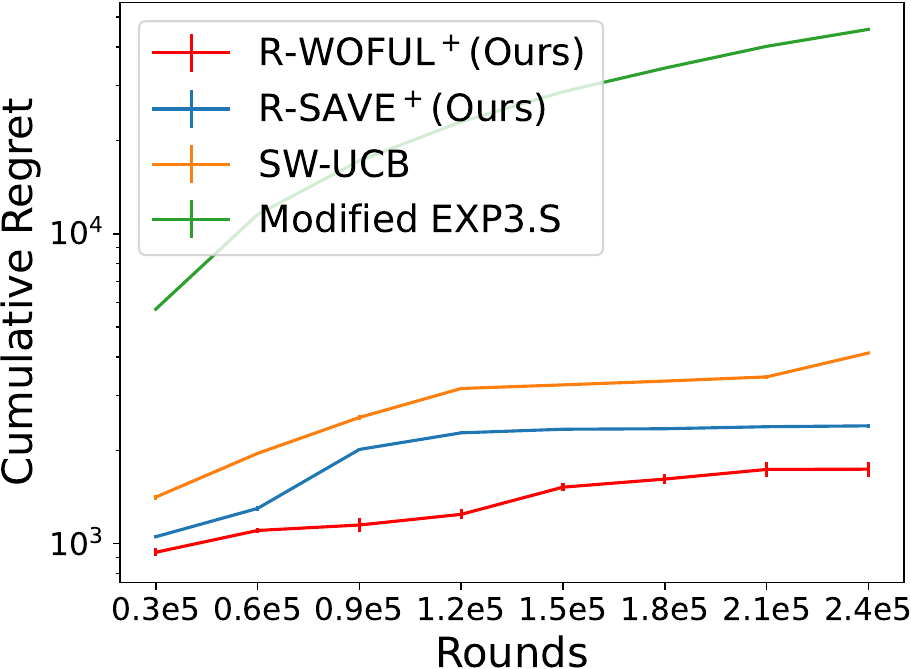}}
    \hfill
    \subfigure[$B_K=K^{1/3}$]{
    \includegraphics[scale=0.24]{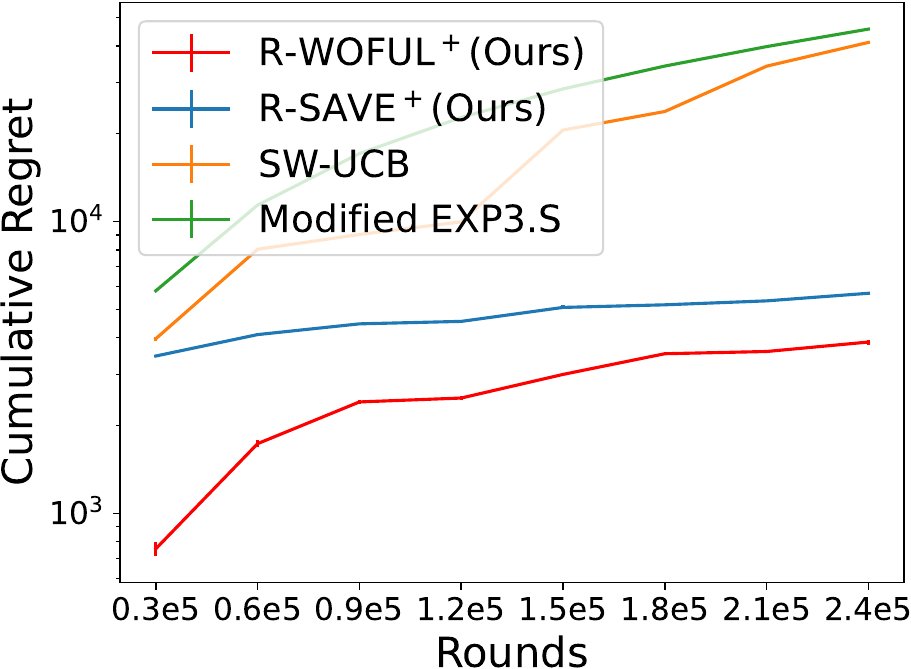}
    }
    \caption{The regret of Restarted-$\algbandit$, $\text{Restarted SAVE}^+$, SW-UCB and Modified EXP3.S under different total rounds.}
    \label{fig:1}
\end{figure*}
\begin{theorem}
    By using the BOB framework with Exp3 as the meta-algorithm and Restarted $\text{SAVE}^+$ as the base algorithm, with the candidate pool $\mathcal{P}$ for Exp3 specified as in Eq.(\ref{bob pool}), Eq.(\ref{w set}), Eq.(\ref{alpha set}), and $H=\lceil d^{\frac{2}{5}}K^{\frac{2}{5}}\rceil$, then the regret of  Restarted $\text{SAVE}^+$-BOB (Algo.\ref{alg:bob algo}) satisfies
    \begin{align}
        \text{Regret}(K) &= \tilde O(d^{4/5}V_K^{2/5}B_K^{1/5}K^{2/5} + d^{2/3}B_K^{1/3}K^{2/3}+d^{2/5}K^{7/10}).
    \end{align}
    \label{theorem bob}
\end{theorem}
\begin{proof}
    See Appendix \ref{app:thirdalg} for the full proof.
\end{proof}

\begin{remark}
We discuss the regret of Algo.\ref{alg:bob algo} in Corollary \ref{corollary w alpha opt} in the following special cases.  In the case where the \emph{total variance} is small, \emph{i.e.}, $V_K=\Tilde{O}(1)$,  assuming that $K^2> d$, our result becomes $\Tilde{O}(d^{2/3}B_K^{1/3}K^{2/3}+d^{1/5}K^{7/10})$, when $d^{14}B_K^{10}>K$, it becomes $\Tilde{O}(d^{2/3}B_K^{1/3}K^{2/3})$, better than all the previous results \cite{cheung2018hedging,zhao2020simple,wang2023revisiting,wei2021non}.
In the worst case where $V_K=O(K)$, our result becomes $\Tilde{O}(d^{4/5}B_K^{1/5}K^{4/5})$.
 
\end{remark}

% \begin{figure*}[h]
%     \centering
%      \subfigure[$B_K=1$]{
%     \includegraphics[scale=0.262]{icml2024/2Regret1.pdf}
%     }
%     \hfill
%      \subfigure[$B_K=10$]{
%     \includegraphics[scale=0.262]{icml2024/2Regret10.pdf}
%     }
%     \hfill
%     \subfigure[$B_K=20$]{
%     \includegraphics[scale=0.262]
%     {icml2024/2Regret20.pdf}}
%     \hfill
%     \subfigure[$B_K=K^{1/3}$]{
%     \includegraphics[scale=0.262]{icml2024/2RegretT3.pdf}
%     }
%     \caption{The regret of Restarted-$\algbandit$, $\text{Restarted SAVE}^+$, SW-UCB and Modified EXP3.S under different total rounds.}
%     \label{fig:1}
% \end{figure*}

\section{Experiments}
\label{sec:experiments}
% \begin{figure}[htp]
%     \centering
%     \includegraphics[scale=0.4]{icml2024/Regret.pdf}
%     \caption{Comparison of Restarted-$\algbandit$ with SW-UCB and Modified EXP3.S when $B_K=1$.} 
%     \label{fig:1}
% \end{figure}
% \begin{figure}[htp]
%     \centering
%     \includegraphics[scale=0.4]{icml2024/RegretT3.pdf}
%     \caption{Comparison of Restarted-$\algbandit$ with SW-UCB and Modified EXP3.S when $B_K=K^{1/3}$.} 
%     \label{fig:2}
% \end{figure}
% \begin{figure}[htp]
%     \centering
%     \includegraphics[scale=0.4]{icml2024/Regret10.pdf}
%     \caption{Comparison of Restarted-$\algbandit$ with SW-UCB and Modified EXP3.S when $B_K=10$.} 
%     \label{fig:3}
% \end{figure}
% \begin{figure}[htp]
%     \centering
%     \includegraphics[scale=0.4]{icml2024/Regret20.pdf}
%     \caption{Comparison of Restarted-$\algbandit$ with SW-UCB and Modified EXP3.S when $B_K=20$.} 
%     \label{fig:4}
% \end{figure}
To validate the effectiveness of our methods, we conduct a series of experiments on the synthetic data.

\noindent\textbf{Problem Setting}
Following the experimental set up in \cite{cheung2019learning}, we consider the 2-armed bandits setting, where the action set $\cD_k = \{(1,0),(0,1)\}$, and 
\begin{align}
\btheta_{k} = \begin{pmatrix}
    0.5 + \frac{3}{10}\sin(5B_{K}\pi k/K)\\
    0.5 + \frac{3}{10}\sin(\pi + 5B_{K}\pi k/K)
\end{pmatrix}.\notag
\end{align}
It is easy to see that the total variation budget can be bounded as $B_K$. At each round $k$, the $\epsilon_{k}$ satisfies the following distribution:
\begin{align}
    \epsilon_k \sim \text{Bernoulli}(0.5/k)-0.5/k.\notag
\end{align}
We can verify that under such a distribution for $\epsilon_k$, the variance of the reward distribution at $k$-th round is $(1-0.5/k)\cdot 0.5/k$, and the total variance $V_K \sim \log K$. 
% The results are shown in Figure.\ref{fig:1}, which are consistent with our theoretical findings. And it's obvious that our algorithms significantly outperform SW-UCB and Modified EXP3.S. The Restarted-$\algbandit$ performs best because it knows the variance and can make more informative decision , the performance of $\text{Restarted SAVE}^+$ is slightly worse than Restarted-$\algbandit$ but still better than the baselines, especially when $B_{K}=K^{1/3}$. 

\noindent\textbf{Baseline algorithms}
We compare the proposed Restarted-$\algbandit$ and $\text{Restarted SAVE}^+$ with SW-UCB \cite{cheung2019learning} and Modified EXP3.S \cite{besbes2014optimal}. For Restarted-$\algbandit$, we set $\lambda=1$, $\hat\beta_{k}=10$, $w=1000$, and we grid search the variance parameters $\alpha$ and $\gamma$, both among values [1, 1.5, 2, 2.5, 3]. Finally we set $\alpha=1$, and $\gamma=2$. For $\text{Restarted SAVE}^+$ we set $w=1000$, $\hat\beta_{k,\ell} = 2^{-\ell+1}$, and grid search $L$ from 1 to 10 with stepsize of 1 and finally choose $L=6$. For SW-UCB, we set $\lambda=1$, $w=1000$, $\beta_{k}=10$. The Modified EXP3.S requires two parameters $\Bar{\alpha}$ and $\Bar{\gamma}$, and we set $\Bar{\gamma}=0.01$ and $\Bar{\alpha}=\frac{1}{K}$.

To test the algorithms' performance under different total time horizons, we let $K$ vary from $3 \times 10^4$ to $2.4 \times 10^5$, with a stepsize of $3 \times 10^4$, and plot the cumulative regret $\text{Regret}(K)$ for these different total time step $K$. We set $B_K=1, 10, 20, \text{and } K^{1/3}$ to observe their performance in different levels of $B_{K}$.

\noindent\textbf{Result}
We plot the results in Figure.\ref{fig:1}, where all the empirical results are averaged over ten independent trials and the error bar is the standard error divided by $\sqrt{10}$. The results are consistent with our theoretical findings. It is evident that our algorithms significantly outperform both SW-UCB and Modified EXP3.S. Among our proposed algorithms, Restarted-$\algbandit$ achieves the best performance. This can be attributed to the fact that it knows the variance and can make more informed decisions. Although $\text{Restarted SAVE}^+$ performed slightly worse than Restarted-$\algbandit$, it still outperforms the baseline algorithms, particularly when $B_{K}=K^{1/3}$. These results highlight the superiority of our methods.

\section{Conclusion and Future Work}
We study non-stationary stochastic linear bandits in this work. We propose Restarted-$\algbandit$ and Restarted SAVE$^+$, two algorithms that utilize the dynamic variance information of the dynamic reward distribution. We show that both of our algorithms are able to achieve better dynamic regret compared with best existing results \citep{wei2021non} under several parameter regimes, \emph{e.g.}, when the total variance $V_K$ is small. Experiment results backup our theoretical claim. It is worth noting there still exist gaps between our current obtained regret and the lower bound \citep{wei2016tracking}, and to fix such a gap leaves as our future work.

%%%%%%%%%%%%%%%%%%%%%%%%%%%%%%%%%%%%%%%%%%%%%%%%%%%%%%%%%%%%%%%%%%%%%%%%%%%%%%%
%%%%%%%%%%%%%%%%%%%%%%%%%%%%%%%%%%%%%%%%%%%%%%%%%%%%%%%%%%%%%%%%%%%%%%%%%%%%%%%
% APPENDIX
%%%%%%%%%%%%%%%%%%%%%%%%%%%%%%%%%%%%%%%%%%%%%%%%%%%%%%%%%%%%%%%%%%%%%%%%%%%%%%%
%%%%%%%%%%%%%%%%%%%%%%%%%%%%%%%%%%%%%%%%%%%%%%%%%%%%%%%%%%%%%%%%%%%%%%%%%%%%%%%

\appendix

\section{$\text{Restarted SAVE}^+$-BOB}\label{app:bob}
In this section, we provide the details of our proposed $\text{Restarted SAVE}^+$-BOB algorithm. The $\text{Restarted SAVE}^+$-BOB algorithm is summarized in Algo.\ref{alg:bob algo}. In this algorithm, we divide the total time rounds $K$ into $\lceil\frac{K}{H}\rceil$ blocks, with each block having $H$ rounds (except the last block). The algorithm also labels all the $\left|\mathcal{P}\right|=\big(\lceil\frac{1}{3}\log_2 K\rceil+\lceil\frac{2}{5}\log_2 K\rceil+2\big)\cdot\big(\lceil\frac{1}{3}\log_2 K\rceil+\lceil\frac{11}{30}\log_2 K\rceil+2\big)$ candidate pairs of parameters in $\mathcal{P}, $\emph{i.e.}, $\mathcal{P}=\{(w_i,\alpha_i)\}_{i=1}^{\left|\mathcal{P}\right|}$. The algorithm initializes $\{s_{j, 1}\}^{\left|\mathcal{P}\right|}_{j=1}$ to be $s_{j,1}=1,\quad\forall j=0,1,\ldots,\left|\mathcal{P}\right|$, which means that at the beginning, the algorithm selects a pair from $\mathcal{P}$ uniformly at random. At the beginning of each block $i\in[\lceil K/H\rceil]$, the meta-learner (Exp3) calculates the distribution $(p_{j, i})^{\left|\mathcal{P}\right|}_{j=1}$ over the candidate set $\mathcal{P}$ by 
\begin{align}
    p_{j, i}=(1-\gamma)\frac{s_{j,i}}{\sum_{u=1}^{\left|\mathcal{P}\right|}s_{u,i}}+\frac{\gamma}{\left|\mathcal{P}\right|+1},\quad\forall j=1,\ldots,\left|\mathcal{P}\right|\,,
\end{align}
where $\gamma$ is defined as 
\begin{align}
    \gamma=\min\left\{1,\sqrt{\frac{(\left|\mathcal{P}\right|+1)\ln(\left|\mathcal{P}\right|+1)}{(e-1)\lceil K/H\rceil}}\right\}\,.
\end{align}
Then, the meta-learner draws a $j_i$ from the distribution $(p_{j, i})^{\left|\mathcal{P}\right|}_{j=1}$, and sets the pair of parameters in block $i$ to be $(w_{j_i},\alpha_{j_i})$, and runs the base algorithm Algo.\ref{alg:1} from scratch in this block with $(w_{j_i},\alpha_{j_i})$, then feeds the cumulative reward in the block $\sum_{k=(i-1)H+1}^{\min\{i\cdot H, K\}}r_k$ to the meta-learner. The meta-learner rescales $\sum_{k=(i-1)H+1}^{\min\{i\cdot H, K\}}r_k$ to $\frac{\sum_{k=(i-1)H+1}^{\min\{i\cdot H, K\}}r_k}{ H+R\sqrt{\frac{H}{2} \log\big(K(\frac{K}{H}+1)\big)} + \frac{2}{3} \cdot R \log\big(K(\frac{K}{H}+1)\big)}$ to make it in the range $[0,1]$ with high probability (supported by Lemma \ref{lemma:bob}). The meta-learner updates the parameter $s_{j_i,i+1}$ to be
\begin{align}
    s_{j_i,i+1}=s_{j_i,i}\cdot\exp\left(\frac{\gamma}{(\left|\mathcal{P}\right|+1)p_{j_i,i}}\left(\frac{1}{2}+\frac{\sum_{k=(i-1)H+1}^{\min\{i\cdot H, K\}}r_k}{ H+R\sqrt{\frac{H}{2} \log\big(K(\frac{K}{H}+1)\big)} + \frac{2}{3} \cdot R \log\big(K(\frac{K}{H}+1)\big)}\right)\right)\,,
\end{align}
and keep others unchanged, \emph{i.e.}, $s_{u,i+1}=s_{u,i}, ~\forall u\neq j_i$. After that, the algorithm will go to the next block, and repeat the same process in block $i+1$.
\begin{algorithm*}[t!] 
    \caption{$\text{Restarted SAVE}^+$-BOB} \label{alg:bob algo}
    \begin{algorithmic}[1]
        \REQUIRE total time rounds $K$; problem dimension $d$; noise upper bound $R$; $\alpha > 0$; the upper bound on the $\ell_2$-norm of $\ab$ in $\cD_k (k\ge 1)$, i.e., $A$; the upper bound on the $\ell_2$-norm of $\btheta_k$ $(k\ge 1)$, i.e., $\pnorm$.
    \STATE Initialize $H=\lceil d^{\frac{2}{5}}K^{\frac{2}{5}}\rceil$; $\mathcal{P}$ as defined in Eq.(\ref{bob pool}), and index the $\left|\mathcal{P}\right|=\big(\lceil\frac{1}{3}\log_2 K\rceil+\lceil\frac{2}{5}\log_2 K\rceil+2\big)\cdot\big(\lceil\frac{1}{3}\log_2 K\rceil+\lceil\frac{11}{30}\log_2 K\rceil+2\big)$ items in $\mathcal{P}$, \emph{i.e.}, $\mathcal{P}=\{(w_i,\alpha_i)\}_{i=1}^{\left|\mathcal{P}\right|}$; $\gamma=\min\left\{1,\sqrt{\frac{(\left|\mathcal{P}\right|+1)\ln(\left|\mathcal{P}\right|+1)}{(e-1)\lceil K/H\rceil}}\right\}$; $\{s_{j, 1}\}^{\left|\mathcal{P}\right|}_{j=1}$ is set to $s_{j,1}=1,\quad\forall j=0,1,\ldots,\left|\mathcal{P}\right|$.
    \FOR{$i=1,2,\ldots,\lceil K/H\rceil$}
    \STATE Calculate the distribution $(p_{j, i})^{\left|\mathcal{P}\right|}_{j=1}$ by $p_{j, i}=(1-\gamma)\frac{s_{j,i}}{\sum_{u=1}^{\left|\mathcal{P}\right|}s_{u,i}}+\frac{\gamma}{\left|\mathcal{P}\right|+1},\quad\forall j=1,\ldots,\left|\mathcal{P}\right|$. 
\STATE Set $j_i\leftarrow j$ with probability $p_{j,i}$, and $(w_i, \alpha_i)\leftarrow (w_{i_i}, \alpha_{j_i})$.
\STATE Run Algo.\ref{alg:1} from scratch in block $i$ (\emph{i.e.}, in rounds $k=(i-1)H+1,\ldots,\min\{i\cdot H, K\}$) with $(w, \alpha)=(w_i, \alpha_i)$.
\STATE Update $s_{j_i,i+1}=s_{j_i,i}\cdot\exp\left(\frac{\gamma}{(\left|\mathcal{P}\right|+1)p_{j_i,i}}\left(\frac{1}{2}+\frac{\sum_{k=(i-1)H+1}^{\min\{i\cdot H, K\}}r_k}{ H+R\sqrt{\frac{H}{2} \log\big(K(\frac{K}{H}+1)\big)} + \frac{2}{3} \cdot R \log\big(K(\frac{K}{H}+1)\big)}\right)\right)$, and keep all the others unchanged, \emph{i.e.}, $s_{u,i+1}=s_{u,i}, ~\forall u\neq j_i$.
    \ENDFOR
    \end{algorithmic}
\end{algorithm*}

\section{Proof of Lemma \ref{lemma:key}}\label{app:keylemma}
For simplicity, we denote 
\begin{align}
    \hat\beta& := 12\sqrt{d\log(1+\frac{wA^2}{\alpha^2d\lambda})\log(32(\log(\frac{\gamma^2}{\alpha}+1)\frac{w^2}{\delta})}  + 30\log(32(\log(\frac{\gamma^2}{\alpha})+1)\frac{w^2}{\delta})\frac{R}{\gamma^2}+ \sqrt{\lambda}\pnorm.
\label{eq:defbanditbeta1}
\end{align}
It is obvious that $\hat\beta\geq\hat\beta_k$ for all $k\in[K]$. We call the restart time rounds \emph{grids} and denote them by $g_1, g_2, \ldots g_{\lceil \frac{K}{w}\rceil-1}$, where $g_i\%w=0$ for all $i\in[\lceil \frac{K}{w}\rceil-1]$. Let $i_k$ be the grid index of time round $k$, \emph{i.e.}, $g_{i_k}\leq k<g_{i_k+1}$.  

For ease of exposition and without loss of generality, we prove the lemma for $k\in[1,w]$. We calculate the estimation difference $|\ba^\top(\hat\btheta_k-\btheta_k)|$ for any $\ab\in\RR^d$, $\|\ab\|_2\leq A$, $k\in[1,w]$. 
By definition:
\begin{equation}
\hat\btheta_k=\hat\bSigma_k^{-1}\bb_k=\hat\bSigma_k^{-1} (\sum_{t=1}^{k-1}\frac{r_t\ab_t}{\bar\sigma_t^2})=\hat\bSigma_k^{-1} (\sum_{t=1}^{k-1}\frac{\ab_t\ab_t^\top\btheta_t}{\bar\sigma_t^2}+\sum_{t=1}^{k-1}\frac{\ab_t\epsilon_t}{\bar\sigma_t^2})\,, 
\end{equation}
where $\hat\Sigma_k=\lambda\bI+\sum_{t=g_{i_k}}^{k-1}\frac{\ab_t\ab_t^\top}{\bar\sigma_t^2}$.

Then we have
\begin{equation}
    \hat\btheta_k-\btheta_k=\hat\bSigma_k^{-1} (\sum_{t=1}^{k-1}\frac{\ab_t\ab_t^\top}{\bar\sigma_t^2}(\btheta_t-\btheta_k)+\sum_{t=1}^{k-1}\frac{\ab_t\epsilon_t}{\bar\sigma_t^2})-\lambda \hat\bSigma_k^{-1}\btheta_k\,.
\end{equation}

Therefore
\begin{equation}
    |\ab^\top(\hat\btheta_k-\btheta_k)|\leq\left|\ab^\top\hat\bSigma_k^{-1}\sum_{t=1}^{k-1}\frac{\ab_t\ab_t^\top}{\bar\sigma_t^2}(\btheta_t-\btheta_k)\right|+\|\ab\|_{\hat\bSigma_k^{-1}}\|\sum_{t=1}^{k-1}\frac{\ab_t\epsilon_t}{\bar\sigma_t^2}\|_{\hat\bSigma_k^{-1}}+\lambda \|\ab\|_{\hat\bSigma_k^{-1}}\|\hat\bSigma_k^{-\frac{1}{2}}\btheta_k\|_2\,,\label{l2 bound difference 1 1}
\end{equation}
where we use the Cauchy-Schwarz inequality.

% \noindent\textbf{some new thgouths} we have
% \begin{align}
%     \left|\ab^\top\hat\bSigma_k^{-1}\sum_{t=\max\{1,k-w\}}^k\frac{\ab_t\ab_t^\top}{\bar\sigma_t^2}(\btheta_t-\btheta_k)\right|
%          & =     \left|\ab^\top\hat\bSigma_k^{-1/2}\cdot\sum_{t=\max\{1,k-w\}}^k\hat\bSigma_k^{-1/2}\frac{\ab_t\ab_t^\top}{\bar\sigma_t^2}(\btheta_t-\btheta_k)\right|\notag \\
%          & \leq \|\ab^\top\hat\bSigma_k^{-1/2}\|_2\cdot\|\sum_{t=\max\{1,k-w\}}^k\hat\bSigma_k^{-1/2}\frac{\ab_t\ab_t^\top}{\bar\sigma_t^2}(\btheta_t-\btheta_k)\|\notag \\
%          & \leq A\cdot \|\ab^\top\hat\bSigma_k^{-1/2}\|_2\cdot \sum_{t=\max\{1,k-w\}}^k\|\hat\bSigma_k^{-1/2}\frac{\ab_t}{\bar\sigma_t^2}\|_2\cdot \|\btheta_t-\btheta_k\|_2\notag \\
%          & \leq A\|\ab^\top\hat\bSigma_k^{-1/2}\|_2\sum_{t=\max\{1,k-w\}}^k\|\hat\bSigma_t^{-1/2}\ab_t\|_2/{\bar\sigma_t^2}\cdot \|\btheta_t-\btheta_k\|_2\notag \\
%          & \leq A\|\ab^\top\hat\bSigma_k^{-1/2}\|_2\sum_{t=\max\{1,k-w\}}^k\|\hat\bSigma_t^{-1/2}\ab_t\|_2/(\gamma^2\|\hat\bSigma_t^{-1/2}\ab_t\|_2)\cdot \|\btheta_t-\btheta_k\|_2\notag \\
%          & = A/\gamma^2\cdot \|\ab\|_{\hat\bSigma_k^{-1}}\sum_{t=\max\{1,k-w\}}^k\|\btheta_t-\btheta_k\|_2
% \end{align}

For the first term, we have that for any $k\in[1,w]$
\begin{align}
    \left|\ab^\top\hat\bSigma_k^{-1}\sum_{t=1}^k\frac{\ab_t\ab_t^\top}{\bar\sigma_t^2}(\btheta_t-\btheta_k)\right|
         & \leq \sum_{t=1}^{k-1} |   \ab^{\top} \hat\bSigma_k^{-1} \frac{\ab_t}{\bar\sigma_t} | \cdot | \frac{\ab_t}{\bar\sigma_t}^\top (\sum_{s = t}^{k-1} (\btheta_{s} - \btheta_{s+1})) | \tag{triangle inequality } \\
     & \leq \sum_{t=1}^{k-1}  |  \ab^{\top} \hat\bSigma_k^{-1} \frac{\ab_t}{\bar\sigma_t} | \cdot \| \frac{\ab_t}{\bar\sigma_t}\|_2 \cdot \| \sum_{s = t}^{k-1} (\btheta_{s} - \btheta_{s+1})\|_2 \tag{Cauchy-Schwarz}
     \\
     & \leq \frac{A}{\alpha}\sum_{t=1}^{k-1}   |  \ab^{\top} \hat\bSigma_k^{-1} \frac{\ab_t}{\bar\sigma_t} | \cdot   \| \sum_{s = t}^{k-1} (\btheta_{s} - \btheta_{s+1})\|_2 \tag{$\| \ab_t\| \leq A$, $\bar\sigma_t\geq\alpha$} \\
    %  \end{align*} 
    %  \begin{align*}
     & \leq \frac{A}{\alpha}\sum_{s =1}^{k-1} \sum_{t = 1}^{s} | \ab^{\top}\hat\bSigma_k^{-1} \frac{\ab_t}{\bar\sigma_t} | \cdot \| \btheta_{s} - \btheta_{s+1}\|_2
     \tag{$\sum_{t =1}^{k-1} \sum_{s=t}^{k-1} =  \sum_{s =1}^{k-1} \sum_{t =1}^{s} $} \\
     & \leq \frac{A}{\alpha}\sum_{s = 1}^{k-1} \sqrt{ \bigg[ \sum_{t =1}^{s} \ab^\top \hat\bSigma_k^{-1} \ab \bigg]  \cdot \biggl [ \sum_{t =1}^{s} \frac{\ab_t}{\bar\sigma_t}^\top \hat\bSigma_k^{-1}  \frac{\ab_t}{\bar\sigma_t}\bigg] }
     \cdot \norm{ \btheta_{s} - \btheta_{s+1}}_2
     \tag{Cauchy-Schwarz}
     \\ & \leq \frac{A}{\alpha}\sum_{s =1}^{k-1} \sqrt{ \bigg[ \sum_{t =1}^{s} \ab^\top \hat\bSigma_k^{-1}\ab \bigg] \cdot d }
     \cdot \norm{ \btheta_{s} - \btheta_{s+1}}_2
     \tag{$(\star)$} \\
     & \leq \frac{A\|\ab\|_2}{\alpha}\sqrt{d} \sum_{s =1}^{k-1} \sqrt{ \frac{\sum_{t =1}^{k-1} 1 }{ \lambda }} \cdot \norm{ \btheta_{s} - \btheta_{s+1}}_2  \tag{$\lambda_{\max}(\hat\bSigma_k^{-1}) \leq \frac{1}{\lambda}$} \\
     & \leq \frac{A^2}{\alpha}\sqrt{\frac{d w}{\lambda }} \sum_{s =1}^{k-1} \norm{ \btheta_{s} - \btheta_{s+1}}_2\,,
\end{align}
where the inequality $(\star)$ follows from the fact that $\sum_{t =1}^{s} \frac{\ab_t}{\bar\sigma_t}^\top \hat\bSigma_k^{-1}  \frac{\ab_t}{\bar\sigma_t}\leq d$ that can be proved as follows. We have $\sum_{t =1}^{k-1} \frac{\ab_t}{\bar\sigma_t}^\top \hat\bSigma_k^{-1}  \frac{\ab_t}{\bar\sigma_t}= \sum_{t =1}^{k-1} \text{tr}\left( \frac{\ab_t}{\bar\sigma_t}^\top \hat\bSigma_k^{-1}  \frac{\ab_t}{\bar\sigma_t}\right) = \text{tr}\left( \hat\bSigma_k^{-1}\sum_{t =1}^{k-1} \frac{\ab_t}{\bar\sigma_t}   \frac{\ab_t}{\bar\sigma_t}^\top \right)$. Given the eigenvalue decomposition $\sum_{t =1}^{k-1}  \frac{\ab_t}{\bar\sigma_t}   \frac{\ab_t}{\bar\sigma_t}^\top = \text{diag}(\lambda_1, \ldots, \lambda_d)^\top$, we have $\hat\bSigma_k = \text{diag}(\lambda_1 + \lambda, \ldots, \lambda_d + \lambda)^\top$, and $\text{tr}\left( \hat\bSigma_k^{-1}\sum_{t =1}^{k-1} \frac{\ab_t}{\bar\sigma_t}   \frac{\ab_t}{\bar\sigma_t}^\top \right) = \sum_{i=1}^d \frac{\lambda_j}{\lambda_j + \lambda } \leq d$.
% \begin{equation}
%     \|\hat\bSigma_k^{-1} \sum_{t=\max\{1,k-w\}}^k\frac{\ba_t\ba_t^\top}{\bar\sigma_t^2}(\btheta_t-\btheta_k)\|_2\leq \sqrt{\frac{dw}{\lambda}}\sum_{t=\max\{1,k-w\}}^k \|\btheta_t-\btheta_{t+1}\|_2\,.
% \end{equation}

For the second term, by the assumption on $\epsilon_k$, we know that
\begin{align}
    &|\epsilon_k/\bar\sigma_k| \leq R/\alpha,\notag\\
    &|\epsilon_k/\bar\sigma_k|\cdot\min\{1,\|\ab_k/\bar\sigma_k\|_{\hat\bSigma_k^{-1}}\} \leq R\|\ab_k\|_{\hat\bSigma_k^{-1}}/\bar\sigma_k^2 \leq R/\gamma^2,\notag \\
    &\EE[\epsilon_k|\ab_{1:k}, \epsilon_{1:k-1}] = 0,\ \EE [(\epsilon_k/\bar\sigma_k)^2|\ab_{1:k}, \epsilon_{1:k-1}] \leq 1,\ \|\ab_k/\bar\sigma_k\|_2 \leq A/\alpha,\notag
\end{align}
Therefore, setting $\cG_k = \sigma(\ab_{1:k}, \epsilon_{1:k-1})$, 
and using that $\sigma_k$ is $\cG_k$-measurable, applying
Theorem \ref{lemma:concentration_variance} to $(\bx_k,\eta_k)=(\ba_k/\bar\sigma_k,\epsilon_k/\bar\sigma_k)$ with $\epsilon = R/\gamma^2$ , we get that
with probability at least $1-\delta$, for all $k \in[1,w]$, 
\begin{equation}
    \|\sum_{t=1}^{k-1}\frac{\ab_t\epsilon_t}{\bar\sigma_t^2}\|_{\hat\bSigma_k^{-1}}\leq 12\sqrt{d\log(1+\frac{(k\%w)A^2}{\alpha^2d\lambda})\log(32(\log(\frac{\gamma^2}{\alpha}+1)\frac{(k\%w)^2}{\delta})}+ 30\log(32(\log(\frac{\gamma^2}{\alpha})+1)\frac{(k\%w)^2}{\delta})\frac{R}{\gamma^2}.
\end{equation}
For the last term
\begin{equation}
\lambda \|\ab\|_{\hat\bSigma_k^{-1}}\|\hat\bSigma_k^{-\frac{1}{2}}\btheta_k\|_2\leq\lambda \|\ab\|_{\hat\bSigma_k^{-1}}\|\hat\bSigma_k^{-\frac{1}{2}}\|_2\|\btheta_k\|_2\leq \lambda \|\ab\|_{\hat\bSigma_k^{-1}}\frac{1}{\sqrt{\lambda_{\text{min}}(\hat\bSigma_k)}}\|\btheta_k\|_2\leq\sqrt{\lambda}B \|\ab\|_{\hat\bSigma_k^{-1}}\,,
\end{equation}
where we use the fact that $\lambda_{\text{min}}(\hat\bSigma_k)\geq\lambda$.

Therefore, with probabilty at least $1-\delta$, we have
\begin{align}
    |\ab^\top(\hat\btheta_k-\btheta_k)|&\leq \frac{A^2}{\alpha}\sqrt{\frac{dw}{\lambda}}\sum_{t=1}^{k-1} \|\btheta_t-\btheta_{t+1}\|_2\notag\\
&\quad+\|\ab\|_{\hat\bSigma_k^{-1}}\bigg(12\sqrt{d\log(1+\frac{(k\%w)A^2}{\alpha^2d\lambda})\log(32(\log(\frac{\gamma^2}{\alpha}+1)\frac{(k\%w)^2}{\delta})} 
    \notag \\
    &\quad+ 30\log(32(\log(\frac{\gamma^2}{\alpha})+1)\frac{(k\%w)^2}{\delta})\frac{R}{\gamma^2} + \sqrt{\lambda}\pnorm\bigg)\notag\\
    % 12\sqrt{d\log(1+\frac{wA^2}{\alpha^2d\lambda})\log(32(\log(\frac{\gamma^2}{\alpha}+1)\frac{w^2}{\delta})}+ 30\log(32(\log(\frac{\gamma^2}{\alpha})+1)\frac{w^2}{\delta})\frac{R}{\gamma^2}+\sqrt{\lambda}B\bigg)\notag\\
    &=\frac{A^2}{\alpha}\sqrt{\frac{dw}{\lambda}}\sum_{t=1}^{k-1} \|\btheta_t-\btheta_{t+1}\|_2+\hat\beta_k\|\ab\|_{\hat\bSigma_k^{-1}}\,,
\end{align}
where $\hat\beta_k$ is defined in Eq.(\ref{eq:defbanditbeta}).

\section{Proof for Theorem \ref{thm: regret for algo1 final}}\label{app:firstalg}

For simplicity of analysis, we only analyze the regret over the first grid, \emph{i.e.}, we try to analyze $\text{Regret}(\tilde K)$ for $\tilde K \in[1,w]$. Denote $\event_{1}$ as the event when Lemma \ref{lemma:key} holds. 
Therefore, under event $\event_{1}$, for any $\tilde K \in[1,w]$, the regret can be bounded by
\begin{align}
    \text{Regret}(\tilde K) &= \sum_{k=1}^{\tilde K}\big[\la \ab_k^*-\ab_k, \btheta_k\ra\big]\notag\\
    &=\sum_{k=1}^{\tilde K}\big[\la \ab_k^*, \btheta_k-\hat\btheta_k\ra+(\la\ab_k^*,\hat\btheta_k\ra+\hat\beta_k\|\ab_k^*\|_{\hat\bSigma_k^{-1}})-(\la\ab_k,\hat\btheta_k\ra+\hat\beta_k\|\ab_k\|_{\hat\bSigma_k^{-1}})+\la \ab_k,\hat\btheta_k-\btheta_k\ra\notag \\
    &\quad+\hat\beta_k\|\ab_k\|_{\hat\bSigma_k^{-1}}-\hat\beta_k\|\ab_k^*\|_{\hat\bSigma_k^{-1}}\big]\notag\\
    &\leq \frac{2A^2}{\alpha}\sqrt{\frac{dw}{\lambda}}\sum_{k=1}^{\tilde K}\sum_{t=1}^{k-1} \|\btheta_t-\btheta_{t+1}\|_2+2\sum_{k=1}^{\tilde K}\min\Big\{1, \hat\beta_k\|\ab_k\|_{\hat\bSigma_k^{-1}}\Big\}\,,\label{regret for restarted woful 1}
\end{align}
where in the last inequality we use the definition of event $\event_1$, the arm selection rule in Line 7 of Algo.\ref{alg:reweightbandit}, and $0 \leq \la \ab_k^*, \btheta^*\ra - \la \ab_k, \btheta^*\ra \leq 2$.

Then we will bound the two terms in Eq.(\ref{regret for restarted woful 1}). 

For the first term, we have
\begin{align}
    &\frac{2A^2}{\alpha}\sqrt{\frac{dw}{\lambda}}\sum_{k=1}^{\tilde K}\sum_{t=1}^{k-1} \|\btheta_t-\btheta_{t+1}\|_2\notag\\
    &=\frac{2A^2}{\alpha}\sqrt{\frac{dw}{\lambda}}\sum_{t=1}^{\tilde K-1} \sum_{k=t}^{\tilde K} \|\btheta_t-\btheta_{t+1}\|_2\notag\\
    &\leq \frac{2A^2}{\alpha}\sqrt{\frac{dw}{\lambda}}w\sum_{t=1}^{\tilde K-1} \|\btheta_t-\btheta_{t+1}\|_2\,.\label{regret bound sum theta 1}
\end{align}

To bound the second term in Eq.(\ref{regret for restarted woful 1}), we decompose the set $[\tilde K]$ into a union of two disjoint subsets $[K] = \cI_1 \cup \cI_2$. 
\begin{align}
    \cI_1 = \Big\{k \in [\tilde K]:\|\frac{\ab_k}{\bar\sigma_k}\|_{\hat\bSigma_k^{-1}} \geq 1 \Big\},\ \cI_2 = \Big\{k \in [\tilde K]:\|\frac{\ab_k}{\bar\sigma_k}\|_{\hat\bSigma_k^{-1}} < 1 \Big\}.
\end{align}
Then the following upper bound of $|\cI_1|$ holds:
\begin{align}
    |\cI_1| &= \sum_{k\in\cI_1}\min\Big\{1, \|\frac{\ab_k}{\bar\sigma_k}\|_{\hat\bSigma_k^{-1}}^2\Big\} \notag\\
    &\leq \sum_{k=1}^{\tilde K}\min\Big\{1, \|\frac{\ab_k}{\bar\sigma_k}\|_{\hat\bSigma_k^{-1}}^2\Big\} \notag\\
    % &=\sum_{i=0}^{\lceil \frac{K}{w}\rceil-1}\sum_{k=i\cdot w+1}^{(i+1)w} \min\Big\{1, \|\frac{\ab_k}{\bar\sigma_k}\|_{\hat\bSigma_k^{-1}}^2\Big\}\notag\\
    % &\leq \sum_{i=0}^{\lceil \frac{K}{w}\rceil-1}\sum_{k=i\cdot w+1}^{(i+1)w} \min\Big\{1, \|\frac{\ab_k}{\bar\sigma_k}\|_{\Tilde\bSigma_k^{-1}}^2\Big\}\label{lemma cheung2019 1 1}\\
    &\leq 2d\iota,\label{eq: bound l1 1}
\end{align}
where $\iota = \log(1+\frac{wA^2}{d\lambda\alpha^2})$, the first equality holds since $\|\frac{\xb_k}{\bar\sigma_k}\|_{\hat\bSigma_k^{-1}} \geq 1$ for $k \in \cI_1$, the last inequality holds due to Lemma \ref{Lemma:abba} together with the fact $\|\frac{\ab_k}{\bar\sigma_k}\|_2 \leq \frac{A}{\alpha}$ since $\bar\sigma_k\geq\alpha$ and $\|\ab_k\|_2\leq A$.

Then, we have
\begin{align}
    &\sum_{k=1}^{\tilde K}\min\Big\{1, \hat\beta_k\|\ab_k\|_{\hat\bSigma_k^{-1}}\Big\}\notag \\
    & =\sum_{k\in \cI_1}\min\Big\{1, \bar\sigma_k\hat\beta_k\|\frac{\ab_k}{\bar\sigma_k}\|_{\hat\bSigma_k^{-1}}\Big\} + \sum_{k\in \cI_2}\min\Big\{1, \bar\sigma_k\hat\beta_k\|\frac{\ab_k}{\bar\sigma_k}\|_{\hat\bSigma_k^{-1}}\Big\}\notag \\
    & \leq \bigg[\sum_{k\in \cI_1} 1\bigg] + \sum_{k\in \cI_2} \bar\sigma_k\hat\beta_k\|\frac{\ab_k}{\bar\sigma_k}\|_{\hat\bSigma_k^{-1}}\notag \\
    & \leq 2d\iota + \hat\beta\sum_{k\in \cI_2} \bar\sigma_k\|\frac{\ab_k}{\bar\sigma_k}\|_{\hat\bSigma_k^{-1}},\label{bound l1 l2}
\end{align}
where the first inequality holds since  $\min\{1,x\}\le 1$ and also $\min\{1,x\}\le x$, 
the second inequality holds by Eq.(\ref{eq: bound l1 1}), and the fact the $\hat\beta\geq\hat\beta_k$ for all $k\in[K]$ ($\hat\beta$ is defined in Eq.(\ref{eq:defbanditbeta1})). Next we further bound the second summation term in \eqref{bound l1 l2}. We decompose $\cI_2 = \cJ_1 \cup \cJ_2$, where
\begin{align}
    &\cJ_1 = \bigg\{k \in \cI_2:\bar\sigma_k = \sigma_k \cup\bar\sigma_k = \alpha \bigg\},\ \cJ_2 = \bigg\{k \in \cI_2:\bar\sigma_k = \gamma\sqrt{\|\ab_k\|_{\hat\bSigma_k^{-1}}}\bigg\}.\notag
\end{align}
Then $\sum_{k\in \cI_2} \bar\sigma_k\|\frac{\ab_k}{\bar\sigma_k}\|_{\hat\bSigma_k^{-1}}=\sum_{k\in \cJ_1} \bar\sigma_k\|\frac{\ab_k}{\bar\sigma_k}\|_{\hat\bSigma_k^{-1}}$+$\sum_{k\in \cJ_2} \bar\sigma_k\|\frac{\ab_k}{\bar\sigma_k}\|_{\hat\bSigma_k^{-1}}$. First, for $k\in \cJ_1$, we have
\begin{align}
    \sum_{k\in \cJ_1} \bar\sigma_k\|\frac{\ab_k}{\bar\sigma_k}\|_{\hat\bSigma_k^{-1}}&\leq \sum_{k\in \cJ_1} (\sigma_k+\alpha)\min\bigg\{1,\|\frac{\ab_k}{\bar\sigma_k}\|_{\hat\bSigma_k^{-1}}\bigg\}\notag\\
    &\leq \sqrt{\sum_{k=1}^{\tilde K}(\sigma_k + \alpha)^2}\sqrt{\sum_{k=1}^{\tilde K} \min\bigg\{1,\|\frac{\ab_k}{\bar\sigma_k}\|_{\hat\bSigma_k^{-1}}\bigg\}^2}\notag\\
        &\leq \sqrt{2\sum_{k=1}^{\tilde K}(\sigma_k^2 + \alpha^2)}\sqrt{\sum_{k=1}^{\tilde K} \min\bigg\{1,\|\frac{\ab_k}{\bar\sigma_k}\|_{\hat\bSigma_k^{-1}}^2\bigg\}}\notag\\
    % &=\sqrt{2\sum_{k=1}^K(\sigma_k^2 + \alpha^2)^2}\sqrt{\sum_{i=0}^{\lceil\frac{K}{w}\rceil-1}\sum_{k=i\cdot w}^{(i+1)w} \min\bigg\{1,\|\frac{\ab_k}{\bar\sigma_k}\|_{\hat\bSigma_k^{-1}}^2\bigg\}}\notag\\
    % &\leq \sqrt{2\sum_{k=1}^K(\sigma_k^2 + \alpha^2)^2}\sqrt{\sum_{i=0}^{\lceil\frac{K}{w}\rceil-1}\sum_{k=i\cdot w}^{(i+1)w} \min\bigg\{1,\|\frac{\ab_k}{\bar\sigma_k}\|_{\Tilde\bSigma_k^{-1}}\bigg\}^2}\label{lemma: chueng 2 1}\\
    &\leq 2\sqrt{\sum_{k=1}^{\tilde K}\sigma_k^2 + {\tilde K}\alpha^2}\sqrt{d\iota}\label{regret j1 1}\,,
\end{align}
where the first inequality holds since $\bar\sigma_k \leq \sigma_k + \alpha$ for $k \in \cJ_1$ and $\|\frac{\ab_k}{\bar\sigma_k}\|_{\hat\bSigma_k^{-1}} \leq 1$ since $k \in \cJ_1 \subseteq \cI_2$, the second inequality holds by Cauchy-Schwarz inequality, the third inequality holds due to $(a+b)^2\leq 2(a^2+b^2)$, and the last inequality holds due to Lemma \ref{Lemma:abba}.

Finally we bound the summation for $k\in\cJ_2$. When $k\in \cJ_2$, we have $\bar\sigma_k=\gamma^2\|\frac{\ab_k}{\bar\sigma_k}\|_{\hat\Sigma_k^{-1}}$. Therefore we have
\begin{align}
    \sum_{k\in \cJ_2} \bar\sigma_k\|\frac{\ab_k}{\bar\sigma_k}\|_{\hat\bSigma_k^{-1}}&=\sum_{k\in \cJ_2} \gamma^2\|\frac{\ab_k}{\bar\sigma_k}\|_{\hat\bSigma_k^{-1}}^2\notag\\
    &\leq\sum_{k=1}^{\tilde K} \gamma^2\min\bigg\{1,\|\frac{\ab_k}{\bar\sigma_k}\|^2_{\hat\bSigma_k^{-1}}\bigg\}\notag\\
    % &=\gamma^2\sum_{i=0}^{\lceil\frac{K}{w}\rceil-1}\sum_{k=i\cdot w}^{(i+1)w}   \min\bigg\{1,\|\frac{\ab_k}{\bar\sigma_k}\|_{\hat\bSigma_k^{-1}}^2\bigg\}\notag\\
    % &\leq \gamma^2\sum_{i=0}^{\lceil\frac{K}{w}\rceil-1}\sum_{k=i\cdot w}^{(i+1)w}   \min\bigg\{1,\|\frac{\ab_k}{\bar\sigma_k}\|_{\Tilde\bSigma_k^{-1}}^2\bigg\}\notag\\
    &\leq 2\gamma^2 d\iota\label{bound j2 1}\,,
\end{align}
where in the first inequality we use the fact that $\|\frac{\ab_k}{\bar\sigma_k}\|_{\hat\bSigma_k^{-1}} \leq 1$ since $k \in \cJ_2 \subseteq \cI_2$, and in the last inequality we use Lemma \ref{Lemma:abba}.

Therefore, with Eq.(\ref{regret for restarted woful 1}), Eq.(\ref{regret bound sum theta 1}), Eq.(\ref{bound l1 l2}), Eq.(\ref{regret j1 1}), Eq.(\ref{bound j2 1}), we can get the regret upper bound for $\tilde K\in[1,w]$
\begin{align}
    \text{Regret}(\tilde K)
    &\leq \frac{2A^2w^{\frac{3}{2}}}{\alpha}\sqrt{\frac{d}{\lambda}}\sum_{k=1}^{\tilde K-1} \|\btheta_k-\btheta_{k+1}\|_2 +4\hat\beta\sqrt{d\iota}\sqrt{\sum_{k\in[\tilde K]}\sigma_k^2 + w\alpha^2}+4d\iota\gamma^2\hat\beta+4d\iota\,.
\end{align}
Therefore, by the same deduction, we can get that
\begin{align}
    \text{Regret}([g_i,g_{i+1}])
    &\leq \frac{2A^2w^{\frac{3}{2}}}{\alpha}\sqrt{\frac{d}{\lambda}}\sum_{k=g_i}^{g_{i+1}-1} \|\btheta_k-\btheta_{k+1}\|_2 +4\hat\beta\sqrt{d\iota}\sqrt{\sum_{k=g_i}^{g_{i+1}}\sigma_k^2 + w\alpha^2}+4d\iota\gamma^2\hat\beta+4d\iota\,,
\end{align}
where we use $\text{Regret}([g_i,g_{i+1}])$ to denote the regret accumulated in the time period $[g_i,g_{i+1}]$.

Finally, without loss of generality, we assume $K\%w=0$. Then we have
\begin{align}
        \text{Regret}(\tilde K)
        &=\sum_{i=0}^{\frac{K}{w}-1}\text{Regret}([g_i,g_{i+1}])\notag\\
            &\leq \frac{2A^2w^{\frac{3}{2}}}{\alpha}\sqrt{\frac{d}{\lambda}}\sum_{i=0}^{\frac{K}{w}-1}\sum_{k=g_i}^{g_{i+1}-1} \|\btheta_k-\btheta_{k+1}\|_2 +4\hat\beta\sqrt{d\iota}\sum_{i=0}^{\frac{K}{w}-1}\sqrt{\sum_{k=g_i}^{g_{i+1}}\sigma_k^2 + w\alpha^2}+\frac{4d\iota\gamma^2\hat\beta K}{w}+\frac{4dK\iota}{w}\notag\\
            &\leq \frac{2A^2w^{\frac{3}{2}}}{\alpha}\sqrt{\frac{d}{\lambda}}\sum_{k=1}^{K-1} \|\btheta_k-\btheta_{k+1}\|_2 +4\hat\beta\sqrt{d\iota}\sqrt{\frac{K}{w}\sum_{i=0}^{\frac{K}{w}-1}(\sum_{k=g_i}^{g_{i+1}}\sigma_k^2 + w\alpha^2)}+\frac{4d\iota\gamma^2\hat\beta K}{w}+\frac{4dK\iota}{w}\notag\\
            &\leq \frac{2A^2w^{\frac{3}{2}}B_K}{\alpha}\sqrt{\frac{d}{\lambda}} +4\hat\beta\sqrt{\frac{Kd\iota}{w}}\sqrt{\sum_{k=1}^{K}\sigma_k^2 + K\alpha^2}+\frac{4d\iota\gamma^2\hat\beta K}{w}+\frac{4dK\iota}{w}\notag\,,
\end{align}
where in the second inequality we use Cauchy-Schwarz inequality, and the last inequality holds due to $\sum_{k\in[K-1]}\|\btheta_k-\btheta_{k+1}\|_2\leq B_K$.

\section{Proof for Theorem \ref{thm:regret1}}\label{app:secalg}
Recall that we call the restart time rounds \emph{grids} and denote them by $g_1, g_2, \ldots g_{\lceil \frac{K}{w}\rceil-1}$, where $g_i\%w=0$ for all $i\in[\lceil \frac{K}{w}\rceil-1]$. Let $i_k$ be the grid index of time round $k$, \emph{i.e.}, $g_{i_k}\leq k<g_{i_k+1}$. We denote $\hat\Psi_{k, \ell}:=\{t: t\in[g_{i_k},k-1], \ell_t=\ell\}$.

For simplicity of analysis, we first try to bound the regret over the first grid, \emph{i.e.}, we try to analyze $\text{Regret}(\tilde K)$ for $\tilde K \in[1,w]$. Note that in this case, for any $k\in[\tilde K]$ with $\tilde K\in[1,w]$, we have $g_{i_k}=1$, so $\hat\Psi_{k, \ell}:=\{t: t\in[1,k-1], \ell_t=\ell\}$.

First, we calculate the estimation difference $|\ba^\top(\hat\btheta_{k,\ell}-\btheta_k)|$ for any $\ab\in\RR^d$, $\|\ab\|_2\leq A$. Recall that by definition, $\hat\Sigma_{k,\ell}=2^{-2\ell}\bI+\sum_{t\in\hat\Psi_{k,\ell}}w_t^2\ab_t\ab_t^\top$, $\hat\bb_{k,\ell}=\sum_{t\in\hat\Psi_{k,\ell}}w_t^2r_t\ab_t$, and
\begin{align}
\hat\btheta_{k,\ell}&=\hat\Sigma_{k,\ell}^{-1}\hat\bb_{k,\ell}=\hat\Sigma_{k,\ell}^{-1}(\sum_{t\in\hat\Psi_{k,\ell}}w_t^2r_t\ab_t)=\hat\Sigma_{k,\ell}^{-1}(\sum_{t\in\hat\Psi_{k,\ell}}w_t^2\ab_t\ab_t^\top\btheta_t+\sum_{t\in\hat\Psi_{k,\ell}}w_t^2\ab_t\epsilon_t)\notag\,.
\end{align}
Then we have 
\begin{align}
    \hat\btheta_{k,\ell}-\btheta_k=\hat\Sigma_{k,\ell}^{-1}(\sum_{t\in\hat\Psi_{k,\ell}}w_t^2\ab_t\ab_t^\top(\btheta_t-\btheta_k)+\sum_{t\in\hat\Psi_{k,\ell}}w_t^2\ab_t\epsilon_t)-2^{-2\ell}\hat\Sigma_{k,\ell}^{-1}\btheta_k\,.
\end{align}
Therefore, we can get
\begin{equation}
    |\ab^\top(\hat\btheta_{k,\ell}-\btheta_k)|\leq\left|\ab^\top\hat\bSigma_{k,\ell}^{-1}\sum_{t\in\hat\Psi_{k,\ell}}w_t^2{\ab_t\ab_t^\top}(\btheta_t-\btheta_k)\right|+\|\ab\|_{\hat\bSigma_{k,\ell}^{-1}}\|\sum_{t\in\hat\Psi_{k,\ell}}w_t^2{\ab_t\epsilon_t}\|_{\hat\bSigma_{k,\ell}^{-1}}+2^{-2\ell} \|\ab\|_{\hat\bSigma_{k,\ell}^{-1}}\|\hat\bSigma_{k,\ell}^{-\frac{1}{2}}\btheta_k\|_2\,,\label{l2 bound difference 12}
\end{equation}
where we use the Cauchy-Schwarz inequality.

For the first term, we have that for any $k\in[1,w]$
\begin{align}  \left|\ab^\top\hat\bSigma_{k,\ell}^{-1}\sum_{t\in\hat\Psi_{k,\ell}}w_t^2{\ab_t\ab_t^\top}(\btheta_t-\btheta_k)\right|
         & \leq \sum_{t\in\hat\Psi_{k,\ell}} |   \ab^{\top} \bSigma_{k,\ell}^{-1} w_t\ab_t | \cdot | w_t\ab_t^\top (\sum_{s = t}^{k-1} (\btheta_{s} - \btheta_{s+1})) | \tag{triangle inequality } \\
     & \leq \sum_{t\in\hat\Psi_{k,\ell}}|  \ab^{\top} \bSigma_{k,\ell}^{-1} w_t{\ab_t} | \cdot \| w_t\ab_t\|_2 \cdot \| \sum_{s = t}^{k-1} (\btheta_{s} - \btheta_{s+1})\|_2 \tag{Cauchy-Schwarz}
     \\
     & \leq {A}\sum_{t\in\hat\Psi_{k,\ell}}   |  \ab^{\top} \hat\bSigma_{k,\ell}^{-1} w_t{\ab_t} | \cdot   \| \sum_{s = t}^{k-1} (\btheta_{s} - \btheta_{s+1})\|_2 \tag{$\| \ab_t\| \leq A$, $w_t=\frac{2^{-\ell_t}}{\|\ab_t\|_{\hat\Sigma_{t,\ell_t}^{-1}}}\leq 1$} \\
    %  \end{align*} 
    %  \begin{align*}
     & \leq {A}\sum_{s=1}^{k-1}\sum_{t\in\hat\Psi_{k,\ell}}   |  \ab^{\top} \hat\bSigma_{k,\ell}^{-1} w_t{\ab_t} | \cdot   \| \btheta_{s} - \btheta_{s+1}\|_2\notag \\
     & \leq {A}\sum_{s =1}^{k-1} \sqrt{ \bigg[ \sum_{t\in\hat\Psi_{k,\ell}} \ab^\top \hat\bSigma_{k,\ell}^{-1} \ab \bigg]  \cdot \biggl [ \sum_{t\in\hat\Psi_{k,\ell}}w_t{\ab_t}^\top \hat\bSigma_{k,\ell}^{-1}  w_t{\ab_t}\bigg] }
     \cdot \norm{ \btheta_{s} - \btheta_{s+1}}_2
     \tag{Cauchy-Schwarz}
     \\ & \leq {A}\sum_{s =1}^{k-1} \sqrt{ \bigg[ \sum_{t\in\hat\Psi_{k,\ell}} \ab^\top \hat\bSigma_{k,\ell}^{-1}\ab \bigg] \cdot d }
     \cdot \norm{ \btheta_{s} - \btheta_{s+1}}_2
     \tag{$(\star)$} \\
     & \leq {A}\|\ab\|_2\sqrt{d} \sum_{s =1}^{k-1} \sqrt{ 2^{2\ell}{\sum_{t \in\hat\Psi_{k,\ell}} 1 }} \cdot \norm{ \btheta_{s} - \btheta_{s+1}}_2  \tag{$\lambda_{\max}(\hat\bSigma_{k,\ell}^{-1}) \leq \frac{1}{2^{-2\ell}}=2^{2\ell}$} \\
     & \leq {A^2}2^{\ell}\sqrt{{d w}} \sum_{s =1}^{k-1} \norm{ \btheta_{s} - \btheta_{s+1}}_2\,,\label{eq: bound sum theta 1}
\end{align}
where the inequality $(\star)$ follows from the fact that $\sum_{t\in\hat\Psi_{k,\ell}}w_t{\ab_t}^\top \hat\bSigma_{k,\ell}^{-1}  w_t{\ab_t}\leq d$ that can be proved as follows. We have $\sum_{t\in\hat\Psi_{k,\ell}}w_t{\ab_t}^\top \hat\bSigma_{k,\ell}^{-1}  w_t{\ab_t}= \sum_{t\in\hat\Psi_{k,\ell}}\text{tr} \left(w_t{\ab_t}^\top \hat\bSigma_{k,\ell}^{-1}  w_t{\ab_t}\right)=\text{tr} \left(\hat\bSigma_{k,\ell}^{-1}  \sum_{t\in\hat\Psi_{k,\ell}}w_t^2{\ab_t} {\ab_t}^\top\right)$. Given the eigenvalue decomposition $\sum_{t\in\hat\Psi_{k,\ell}}w_t^2{\ab_t} {\ab_t}^\top = \text{diag}(\lambda_1, \ldots, \lambda_d)^\top$, we have $\hat\bSigma_{k,\ell} = \text{diag}(\lambda_1 + \lambda, \ldots, \lambda_d + \lambda)^\top$, and $\text{tr} \left(\hat\bSigma_{k,\ell}^{-1}  \sum_{t\in\hat\Psi_{k,\ell}}w_t^2{\ab_t} {\ab_t}^\top\right)= \sum_{i=1}^d \frac{\lambda_j}{\lambda_j + \lambda } \leq d$.

For the second term in Eq.(\ref{l2 bound difference 12}), we can apply Theorem \ref{thm:bernstein1} for the layer $\ell$. In detail, for any $k\in[K]$, for each $t \in \hat\Psi_{k, \ell}$, we have \begin{align} 
        \|w_t \ab_t\|_{\hat\bSigma_{t, \ell}^{-1}} = 2^{-\ell}, \quad
        \EE[w_t^2 \epsilon_t^2| \cF_{t}] \le w_t^2 \EE[\epsilon_t^2| \cF_{t}] \le w_t^2 \sigma_t^2, \quad
        |w_t \epsilon_t| \le |\epsilon_t| \le R, \notag
    \end{align}
    where the last inequality holds due to the fact that
$w_t = \frac{2^{-\ell_t}}{\|\ab_t\|_{\hat\bSigma_{t, \ell_t}^{-1}}}\leq 1$. According to Theorem \ref{thm:bernstein1}, and taking a union bound, we can deduce that with probability at least $1 - \delta$, for all $\ell\in[L]$, for all round $k \in \Psi_{K + 1, \ell}, \ $\begin{align}  \|\sum_{t\in\hat\Psi_{k,\ell}}w_t^2{\ab_t\epsilon_t}\|_{\hat\bSigma_{k,\ell}^{-1}}\leq 16 \cdot 2^{-\ell} \sqrt{\sum_{t \in \hat\Psi_{k, \ell}} w_t^2 \sigma_t^2\log(\frac{4w^2 L }{\delta} )} + 6 \cdot 2^{-\ell} R \log(\frac{4w^2 L }{\delta})\,. \label{event unknown 1}
    \end{align}
    For simplicity, we denote $\cE_{\conf}$ as the event such that Eq.(\ref{event unknown 1}) holds.

For the third term in Eq.(\ref{l2 bound difference 12}), we have 
\begin{equation}
    2^{-2\ell} \|\ab\|_{\hat\bSigma_{k,\ell}^{-1}}\|\hat\bSigma_{k,\ell}^{-\frac{1}{2}}\btheta_k\|_2\leq2^{-2\ell} \|\ab\|_{\hat\bSigma_{k,\ell}^{-1}}\|\hat\bSigma_k^{-\frac{1}{2}}\|_2\|\btheta_k\|_2\leq 2^{-2\ell} \|\ab\|_{\hat\bSigma_{k,\ell}^{-1}}\frac{1}{\sqrt{\lambda_{\text{min}}(\hat\bSigma_{k,\ell})}}\|\btheta_k\|_2\leq2^{-\ell}B \|\ab\|_{\hat\bSigma_k^{-1}}\,,\label{lambda bound}
\end{equation}
where we use the fact that $\lambda_{\text{min}}(\hat\bSigma_{k,\ell})\geq2^{-2\ell}$.

For simplicity, we denote $\ell^* = \lceil \frac{1}{2} \log_2 \log\left(4(w + 1)^2 L /\delta\right) \rceil + 8$. 
Then, under $\event_{\conf}$, by the definition of $\hat \beta_{k, \ell}$ in Eq.(\ref{eq:def:beta}), Lemma \ref{lemma:var} and Lemma \ref{lemma:varest}, with probability at least $1-\delta$, we have for all $\ell^* + 1 \le \ell \le L$, \begin{align}
        \hat\beta_{k, \ell} \ge 16 \cdot 2^{-\ell} \sqrt{\sum_{t \in \hat\Psi_{k, \ell}} w_t^2 \sigma_t^2\log(\frac{4w^2 L }{\delta} )} + 6 \cdot 2^{-\ell} R \log(\frac{4w^2 L }{\delta})+2^{-\ell}B.\label{beta bound}
    \end{align}

Therefore, with Eq.(\ref{l2 bound difference 12}), Eq.(\ref{eq: bound sum theta 1}), Eq.(\ref{event unknown 1}), Eq.(\ref{lambda bound}), Eq.(\ref{beta bound}), with probability at least $1-3\delta$, for all $\ell^* + 1 \le \ell \le L$ we have
\begin{equation}
    |\ab^\top(\hat\btheta_{k,\ell}-\btheta_k)|\leq {A^2}2^{\ell}\sqrt{{d w}} \sum_{s =1}^{k-1} \norm{ \btheta_{s} - \btheta_{s+1}}_2+\hat\beta_{k,\ell}\|\ab\|_{\hat\Sigma_{k,\ell}^{-1}}\label{bound estimation for algo2}\,.
\end{equation}

Then for all $k\in[K]$ such that $\ell^* + 1 \le \ell_k \le L$, with probability at least $1-3\delta$ we have
\begin{align}
   \la\ab_k^*,\btheta_k\ra&\leq \min_{\ell\in[L]}\la \ab_k^*,\hat\btheta_{k,\ell}\ra+{A^2}2^{\ell}\sqrt{{d w}} \sum_{s =1}^{k-1} \norm{ \btheta_{s} - \btheta_{s+1}}_2+\hat\beta_{k,\ell}\|\ab_k^*\|_{\hat\Sigma_{k,\ell}^{-1}}\notag\\
   &\leq {A^2}2^{L}\sqrt{{d w}} \sum_{s =1}^{k-1} \norm{ \btheta_{s} - \btheta_{s+1}}_2+\min_{\ell\in[L]}\la \ab_k^*,\hat\btheta_{k,\ell}\ra+\hat\beta_{k,\ell}\|\ab_k^*\|_{\hat\Sigma_{k,\ell}^{-1}}\notag\\
   &\leq {A^2}2^{L}\sqrt{{d w}} \sum_{s =1}^{k-1} \norm{ \btheta_{s} - \btheta_{s+1}}_2+\min_{\ell\in[L]}\la \ab_k,\hat\btheta_{k,\ell}\ra+\hat\beta_{k,\ell}\|\ab_k\|_{\hat\Sigma_{k,\ell}^{-1}}\notag\\
   &\leq {A^2}2^{L}\sqrt{{d w}} \sum_{s =1}^{k-1} \norm{ \btheta_{s} - \btheta_{s+1}}_2+\la \ab_k,\hat\btheta_{k,\ell_k-1}\ra+\hat\beta_{k,\ell_k-1}\|\ab_k\|_{\hat\Sigma_{k,\ell_k-1}^{-1}}\,,\label{regret upper bound by max min}
\end{align}
where the first inequality holds because of Eq.(\ref{bound estimation for algo2}), the third inequality holds because of the arm selection rule in Line 8 of Algo.\ref{alg:1}.

% \begin{align}
%     regret_{[s,t]} = \sum_s^t \|\theta- \theta\|_2 + \sqrt{\sum_s^t \sigma^2}
% \end{align}

We decompose the regret for $\tilde K\in[1,w]$ as follows
\begin{align}
    \text{Regret}(\tilde K)&=\sum_{k\in[\tilde K]} (\la\ab_k^*,\btheta_k\ra-\la\ab_k,\btheta_k\ra)\notag\\
    &=\sum_{\ell \in [\ell^*]}\sum_{k\in\hat\Psi_{\tilde K+1,\ell}}(\la\ab_k^*,\btheta_k\ra-\la\ab_k,\btheta_k\ra)+\sum_{\ell \in [L]\backslash[\ell^*]}\sum_{k\in\hat\Psi_{\tilde K+1,\ell}}(\la\ab_k^*,\btheta_k\ra-\la\ab_k,\btheta_k\ra)\notag\\
    &\quad+\sum_{k\in\hat\Psi_{\tilde K+1,L+1}}(\la\ab_k^*,\btheta_k\ra-\la\ab_k,\btheta_k\ra)\label{regret decompose of algo2 1}\,.
\end{align}
We will bound the three terms separately. For the first term, we have
for layer $\ell \in [\ell^*]$ and round $k\in\hat\Psi_{\tilde K+1,\ell}$, we have \begin{align} 
            \sum_{k\in\hat\Psi_{\tilde K+1,\ell}} \big(\la \ab_{k}^*, \btheta^* \ra - \la \ab_{k}, \btheta^* \ra\big) &\le 2 \left|\Psi_{K + 1, \ell}\right| \notag\\
            &=  2^{2\ell+1} \sum_{k\in\hat\Psi_{\tilde K+1,\ell}} \|w_k\ab_k\|_{\hat \bSigma_{k, \ell}^{-1}}^2 \notag\\
            &\le 2\cdot128^2\log(\frac{4(w + 1)^2 L }{\delta}) \sum_{k\in\hat\Psi_{\tilde K+1,\ell}} \|w_k\ab_k\|_{\hat \bSigma_{k, \ell}^{-1}}^2\notag\\
            &\leq2\cdot128^2\log(\frac{4(w + 1)^2 L }{\delta})\cdot2d\log(1+\frac{2^{2\ell}wA^2}{d})\notag\\
            &=\Tilde{O}(d)\,,\label{bound small l11}
            \end{align}
       where the first inequality holds because the reward is in $[-1,1]$, the equation follows from the fact that $\|w_k\ab_k\|_{\hat \bSigma_{k, \ell}^{-1}}=2^{-\ell}$ holds for all $k \in \Psi_{K + 1, \ell}$,
        the second inequality holds due to the fact that $2^{\ell^*} \le 128\sqrt{\log(4(w + 1)^2 L /\delta)}$, and the last inequality holds due to Lemma \ref{Lemma:abba}.
        
% We then bound $\sum_{k \in \Psi_{K + 1, \ell}} \|w_k\ab_k\|_{\hat \bSigma_{k, \ell}^{-1}}^2$. We denote $\Psi_{[a,b],\ell}:=\{k:k\in[a,b], \ell_k=\ell\}$. Then we have
% \begin{align}
%     \sum_{k \in \Psi_{K + 1, \ell}} \|w_k\ab_k\|_{\hat \bSigma_{k, \ell}^{-1}}^2&=\sum_{i=0}^{\lceil\frac{K}{w}\rceil-1}\sum_{k\in\Psi_{[i\cdot w, (i+1)w],\ell}} \|w_k\ab_k\|_{\hat \bSigma_{k, \ell}^{-1}}^2\notag\\
%     % &\leq \sum_{i=0}^{\lceil\frac{K}{w}\rceil-1}\sum_{k\in\Psi_{[i\cdot w, (i+1)w],\ell}} \|w_k\ab_k\|_{\Tilde \bSigma_{k, \ell}}^2\notag\\
%     &\leq \frac{K}{w}\cdot 2d\log(1+\frac{2^{2\ell}wA^2}{d})\,,\label{bound sum of elliptical algo2 1}
% \end{align}
% where the last inequality holds due to Lemma \ref{Lemma:abba}.

Therefore
\begin{equation}
    \sum_{\ell \in [\ell^*]}\sum_{k\in\hat\Psi_{\tilde K+1,\ell}}(\la\ab_k^*,\btheta_k\ra-\la\ab_k,\btheta_k\ra)=\tilde{O}(d)\label{bound small l 2}\,.
\end{equation}
For the second part in Eq.(\ref{regret decompose of algo2 1}), we have
\begin{align}
    &\sum_{\ell \in [L]\backslash[\ell^*]}\sum_{k\in\hat\Psi_{\tilde K+1,\ell}}(\la\ab_k^*,\btheta_k\ra-\la\ab_k,\btheta_k\ra)\notag \\
    &\quad \leq  \sum_{\ell \in [L]\backslash[\ell^*]}\sum_{k\in\hat\Psi_{\tilde K+1,\ell}}\bigg(\la\ab_k,\hat\btheta_{k,\ell-1}\ra+\hat\beta_{k,\ell-1}\|\ab_k\|_{\hat\Sigma_{k,\ell-1}^{-1}}\notag\\
    &\quad+{A^2}2^{L}\sqrt{{d w}} \sum_{k\in\hat\Psi_{\tilde K+1,\ell}} \norm{ \btheta_{s} - \btheta_{s+1}}_2-\la\ab_k,\btheta_k\ra\bigg)\notag\\
       &\leq 2\sum_{\ell \in [L]\backslash[\ell^*]}\sum_{k\in\hat\Psi_{\tilde K+1,\ell}} \hat\beta_{k,\ell-1}\|\ab_k\|_{\hat\Sigma_{k,\ell-1}^{-1}}+{A^2}\sqrt{{d w}}\sum_{\ell \in [L]\backslash[\ell^*]}\sum_{k\in\hat\Psi_{\tilde K+1,\ell}}2^{L} \sum_{s=1}^{k-1} \norm{ \btheta_{s} - \btheta_{s+1}}_2\,,\label{bound of large L decom}
\end{align}
where the inequality holds due to Eq.(\ref{regret upper bound by max min}), the second inequality holds due to Eq.(\ref{bound estimation for algo2}). We then try to bound the two terms.

For the first term in Eq.(\ref{bound of large L decom}), we have
\begin{align}
    \sum_{\ell \in [L]\backslash[\ell^*]}\sum_{k\in\hat\Psi_{\tilde K+1,\ell}} \hat\beta_{k,\ell-1}\|\ab_k\|_{\hat\Sigma_{k,\ell-1}^{-1}}
    &\leq \sum_{\ell \in [L]\backslash[\ell^*]}\sum_{k\in\hat\Psi_{\tilde K+1,\ell}} \hat\beta_{k,\ell-1}\cdot2^{-\ell}\notag\\
    &\leq \sum_{\ell \in [L]\backslash[\ell^*]}\hat\beta_{\tilde K,\ell-1}\cdot2^{-\ell}\left|\hat\Psi_{\tilde K+1,\ell}\right|\notag\\
    &=\sum_{\ell \in [L]\backslash[\ell^*]}\hat\beta_{\tilde K,\ell-1}\cdot2^{\ell}\sum_{k\in\hat\Psi_{\tilde K+1,\ell}}\|w_k\ab_k\|_{\Sigma_{k,\ell}^{-1}}^2\notag\\
    &\leq \sum_{\ell \in [L]\backslash[\ell^*]}\hat\beta_{\tilde K,\ell-1}\cdot2^{\ell}\cdot 2d\log(1+\frac{2^{2\ell}\tilde KA^2}{d})\notag\\
    &=\Tilde{O}(d\cdot2^{\ell}\cdot\hat\beta_{\tilde K,\ell-1})\notag\\
    &=\Tilde{O}\bigg(d\big(\sqrt{\sum_{k=1}^{\tilde K}\sigma_k^2}+R+1\big)\bigg)\,,\label{regret algo2 final part2}
    % &\leq\sum_{\ell \in [L]\backslash[\ell^*]}2^{-\ell}\sqrt{\sum_{k\in\Psi_{K+1,\ell}} 1 \cdot\sum_{k\in\Psi_{K+1,\ell}} \hat\beta_{k,\ell-1}^2}\notag\\
    % &=\Tilde{O}\bigg(\sum_{\ell \in [L]\backslash[\ell^*]}2^{-2\ell} \sqrt{\sum_{k\in\Psi_{K+1,\ell}} 1\cdot \sum_{k\in\Psi_{K+1,\ell}} (\sum_{t\in\hat\Psi_{k,\ell}} \sigma_t^2+C_1)}\bigg)\notag\\
    % &\leq \Tilde{O}\bigg(\sum_{\ell \in [L]\backslash[\ell^*]}2^{-2\ell} \sqrt{2^{2\ell}\sum_{k\in\Psi_{K+1,\ell}} \|w_k\ab_k\|_{\hat\Sigma_{k,\ell}^{-1}}^2\cdot \sum_{k\in\Psi_{K+1,\ell}} (\sum_{t\in\hat\Psi_{k,\ell}} \sigma_t^2+C_1)}\bigg)\notag
\end{align}
where the first inequality holds because by the algorithm design, we have for all $k\in\hat\Psi_{\tilde K+1,\ell}$: $\|\ab_k\|_{\hat\Sigma_{k,\ell-1}^{-1}}\leq 2^{-\ell}$; the second inequality holds because for all $k\in\hat\Psi_{\tilde K+1,\ell}$, $\hat\beta_{k,\ell-1}\leq\hat\beta_{\Tilde K,\ell-1}$; the first equality holds because for all $k\in\hat\Psi_{\tilde K+1,\ell}$, $\|w_k\ab_k\|_{\Sigma_{k,\ell}^{-1}}^2=2^{-2\ell}$; the third inequality holds by Lemma \ref{Lemma:abba}; the last two equalities hold because by Lemma \ref{lemma:var} and Lemma \ref{lemma:varest}, we have $\hat\beta_{\Tilde K,\ell-1}=\Tilde{O}\bigg(2^{-\ell}(\sqrt{\sum_{k=1}^{\tilde K}\sigma_k^2}+R+1)\bigg)$.

For the second term in Eq.(\ref{bound of large L decom}), we have
\begin{align}
{A^2}\sqrt{{d w}}\sum_{\ell \in [L]\backslash[\ell^*]}\sum_{k\in\hat\Psi_{\tilde K+1,\ell}}2^{L} \sum_{s=1}^{k-1} \norm{ \btheta_{s} - \btheta_{s+1}}_2&\leq A^22^{L}\sqrt{dw}\sum_{k\in[\tilde K-1]}\sum_{s =1}^{k-1} \norm{ \btheta_{s} - \btheta_{s+1}}_2\notag\\
    &\leq \frac{A^2\sqrt{d}w^{\frac{3}{2}}}{\alpha}\sum_{k=1}^{\Tilde{K}-1}\norm{ \btheta_{k} - \btheta_{k+1}}_2
\end{align}
Therefore, with this, Eq.(\ref{bound of large L decom}), and Eq.(\ref{regret algo2 final part2}), we have
\begin{align}
    \sum_{\ell \in [L]\backslash[\ell^*]}\sum_{k\in\hat\Psi_{\tilde K+1,\ell}}(\la\ab_k^*,\btheta_k\ra-\la\ab_k,\btheta_k\ra)\leq \frac{A^2\sqrt{d}w^{\frac{3}{2}}}{\alpha}\sum_{k=1}^{\Tilde{K}-1}\norm{ \btheta_{k} - \btheta_{k+1}}_2+\Tilde{O}\bigg(d\big(\sqrt{\sum_{k=1}^{\tilde K}\sigma_k^2}+R+1\big)\bigg)\,.\label{regret algo2 result part 2 1 final}
\end{align}
Finally, for the last term in Eq.(\ref{regret decompose of algo2 1}), we have
\begin{align}
    \sum_{k\in\hat\Psi_{\tilde K+1,L+1}}(\la\ab_k^*,\btheta_k\ra-\la\ab_k,\btheta_k\ra)&\leq \sum_{k\in\hat\Psi_{\tilde K+1,L+1}}\bigg(\la\ab_k,\hat\btheta_{k,L}\ra+\hat\beta_{k,L}\|\ab_k\|_{\hat\Sigma_{k,L}^{-1}}\notag \\
    &\quad+A^2 2^L\sqrt{dw}\sum_{s=1}^{k-1} \norm{ \btheta_{s} - \btheta_{s+1}}_2-\la\ab_k,\btheta_k\ra\bigg)\notag\\
    &\leq \sum_{k\in\hat\Psi_{\tilde K+1,L+1}} \bigg(2\hat\beta_{k,L}\|\ab_k\|_{\hat\Sigma_{k,L}^{-1}}+A^2 2^{L+1}\sqrt{dw}\sum_{s=1}^{k-1} \norm{ \btheta_{s} - \btheta_{s+1}}_2\bigg)\notag\\
    &\leq \sum_{k\in\hat\Psi_{\tilde K+1,L+1}} \bigg(2^{-L+1}\hat\beta_{k,L}+A^2 2^{L+1}\sqrt{dw}\sum_{s=1}^{k-1} \norm{ \btheta_{s} - \btheta_{s+1}}_2\bigg)\notag\\
    &\leq \frac{2A^2\sqrt{d}w^{\frac{3}{2}}}{\alpha}\sum_{k=1}^{\tilde K-1} \norm{ \btheta_{k} - \btheta_{k+1}}_2+\sum_{k\in\hat\Psi_{\tilde K+1,L+1}} 2^{-L+1}\hat\beta_{\tilde K,L}\notag\\
    &\leq\frac{2A^2\sqrt{d}w^{\frac{3}{2}}}{\alpha}\sum_{k=1}^{\tilde K-1} \norm{ \btheta_{k} - \btheta_{k+1}}_2+w\cdot2\alpha\cdot\hat\beta_{\tilde K,L}\notag\\
    &=\frac{2A^2\sqrt{d}w^{\frac{3}{2}}}{\alpha}\sum_{k=1}^{\tilde K-1} \norm{ \btheta_{k} - \btheta_{k+1}}_2+\Tilde{O}\bigg(w\alpha^2\cdot\big(\sqrt{\sum_{k=1}^{\tilde K}\sigma_k^2}+R+1\big)\bigg)
    \,,\label{regret part 3 algo2}
\end{align}
where the first inequality holds due to Eq.(\ref{regret upper bound by max min}), the second inequality holds due to Eq.(\ref{bound estimation for algo2}), the third inequality holds because by the algorithm design, we have for all $k\in\hat\Psi_{\tilde K+1,L+1}$: $\|\ab_k\|_{\hat\Sigma_{k,L}^{-1}}\leq 2^{-L}$, the fourth inequality holds due to the same reasons as before, and the fact that $\hat\beta_{\tilde K,L}\geq\hat\beta_{k,L}$ for all $k\in\hat\beta_{\tilde K,L}$; the last inequality holds due to $\hat\beta_{\Tilde K,\ell-1}=\Tilde{O}\bigg(\alpha(\sqrt{\sum_{k=1}^{\tilde K}\sigma_k^2}+R+1)\bigg)$.

Plugging Eq.(\ref{regret algo2 result part 2 1 final}), Eq.(\ref{regret part 3 algo2}), and Eq.(\ref{bound small l 2}) into Eq.(\ref{regret decompose of algo2 1}), we can get that for $\tilde K\in[1,w]$

          \begin{align} 
            \text{Regret}(\tilde K)&= \tilde{O}\bigg(\frac{A^2\sqrt{d}w^{\frac{3}{2}}}{\alpha}\sum_{k=1}^{\Tilde{K}-1}\norm{ \btheta_{k} - \btheta_{k+1}}_2+\big(w\alpha^2+d\big)\cdot\big(\sqrt{\sum_{k=1}^{\tilde K}\sigma_k^2}+R+1\big)\bigg)\,.
        \end{align}
By the same deduction we can get 

          \begin{align} 
            \text{Regret}([g_i,g_{i+1}])&= \tilde{O}\bigg(\frac{A^2\sqrt{d}w^{\frac{3}{2}}}{\alpha}\sum_{k=g_i}^{g_{i+1}}\norm{ \btheta_{k} - \btheta_{k+1}}_2+\big(w\alpha^2+d\big)\cdot\big(\sqrt{\sum_{k=g_i}^{g_{i+1}}\sigma_k^2}+R+1\big)\bigg)\,.
        \end{align}

Finally, without loss of generality, we assume $K\%w=0$. Then we have
\begin{align}
        \text{Regret}(K)
        &=\sum_{i=0}^{\frac{K}{w}-1}\text{Regret}([g_i,g_{i+1}])\notag\\
        &=\tilde{O}\bigg(\frac{A^2\sqrt{d}w^{\frac{3}{2}}}{\alpha}\sum_{i=0}^{\frac{K}{w}-1}\sum_{k=g_i}^{g_{i+1}}\norm{ \btheta_{k} - \btheta_{k+1}}_2+\big(w\alpha^2+d\big)\cdot\sum_{i=0}^{\frac{K}{w}-1}\big(\sqrt{\sum_{k=g_i}^{g_{i+1}}\sigma_k^2}+R+1\big)\bigg)\notag\\
        &\leq\tilde{O}\bigg(\frac{A^2\sqrt{d}w^{\frac{3}{2}}}{\alpha}\sum_{k=1}^{K-1}\norm{ \btheta_{k} - \btheta_{k+1}}_2+\big(w\alpha^2+d\big)\cdot\big(\sqrt{\frac{K}{w}\sum_{i=0}^{\frac{K}{w}-1}\sum_{k=g_i}^{g_{i+1}}\sigma_k^2}+\frac{KR}{w}+\frac{K}{w}\big)\bigg)\notag\\
        &\leq\tilde{O}\bigg(\frac{A^2\sqrt{d}w^{\frac{3}{2}}B_K}{\alpha}+\big(w\alpha^2+d\big)\cdot\sqrt{\frac{K}{w}\sum_{k=1}^{K}\sigma_k^2}+\big(1+R\big)\cdot\big(K\alpha^2+\frac{Kd}{w}\big)\bigg)\notag\,,
        \end{align}
        where the first inequality holds due to the Cauchy-Schwarz inequality, the last inequality holds because $\sum_{k=1}^{K-1}\norm{ \btheta_{k} - \btheta_{k+1}}_2\leq B_K$.
\section{Proof of Theorem \ref{theorem bob}}\label{app:thirdalg}
With the candidate pool set $\mathcal{P}$ designed as in Eq.(\ref{bob pool}), Eq.(\ref{w set}), Eq.(\ref{alpha set}), and $H=\lceil d^{\frac{2}{5}}K^{\frac{2}{5}}\rceil$, we have $\left|\mathcal{P}\right|=O(\log K)$, and for any $w\in \mathcal{W}$, $w\leq H$.

We denote the optimal $(w,\alpha)$ with the knowledge of $V_K$ and $B_K$ in Corollary \ref{corollary w alpha opt} as $(w^*,\alpha^*)$. We denote the best approximation of $(w^*,\alpha^*)$ in the candidate set $\mathcal{P}$ as $(w^+,\alpha^+)$. Then we can decompose the regret as follows
\begin{align}
    \text{Regret}(K)=\sum_{k=1}^{K}\la\ab_t^*,\btheta_k\ra-\la\ab_t,\btheta_k\ra&=\underbrace{\sum_{k=1}^K\la\ab_t^*,\btheta_k\ra-\sum_{i=1}^{\lceil
    \frac{K}{H}\rceil}\sum_{k=(i-1)H+1}^{iH}\la\ab_t(w^+,\alpha^+),\btheta_k\ra}_{(1)}\notag\\
    &+\underbrace{\sum_{i=1}^{\lceil
    \frac{K}{H}\rceil}\sum_{k=(i-1)H+1}^{iH}\la\ab_t(w^+,\alpha^+),\btheta_k\ra-\la\ab_t(w_i,\alpha_i),\btheta_k\ra}_{(2)}\,.
\end{align}
The first term (1) is the dynamic regret of $\text{Restarted SAVE}^+$ with the best parameters in the candidate pool $\mathcal{P}$. The second term (2) is the regret overhead of meta-algorithm due to adaptive exploration of unknown optimal parameters.

By the design of the candidate pool set $\mathcal{P}$ in Eq.(\ref{bob pool}), Eq.(\ref{w set}), Eq.(\ref{alpha set}), we have that there exists a pair $(w^+,\alpha^+)\in\mathcal{P}$ such that $w^+<w^*<2w^+$, and $\alpha^+<\alpha^*<2\alpha^+$. Therefore, employing the regret bound in Theorem \ref{thm:regret1}, we can get
\begin{align}
    (1)&\leq \sum_{i=1}^{\lceil\frac{K}{H}\rceil}\Tilde O(\sqrt{d}w^{+1.5}B_i/\alpha^+ + \alpha^{+2}(H+\sqrt{w^+HV_i})
         + d\sqrt{HV_i/w^+} + dH/w^+)\notag\\
     &\leq \Tilde O(\sqrt{d}w^{+1.5}B_K/\alpha^+ + \alpha^{+2}(K+\sqrt{w^+H\frac{K}{H}\sum_{i=1}^{\lceil\frac{K}{H}\rceil}V_i})
         + d\sqrt{H\frac{K}{H}\sum_{i=1}^{\lceil\frac{K}{H}\rceil}V_i/w^+} + dK/w^+)\notag\\
         &=\Tilde O(\sqrt{d}w^{+1.5}B_K/\alpha^+ + \alpha^{+2}(K+\sqrt{w^+KV_K})
         + d\sqrt{KV_K/w^+} + dK/w^+)\notag\\
         &=\Tilde O(\sqrt{d}w^{*1.5}B_K/\alpha^* + \alpha^{*2}(K+\sqrt{w^*KV_K})
         + d\sqrt{KV_K/w^*} + dK/w^*)\notag\\
         &=\tilde O(d^{4/5}V_K^{2/5}B_K^{1/5}K^{2/5} + d^{2/3}B_K^{1/3}K^{2/3})\,,
\end{align}
where we denote $B_i$ as the total variation budget in block $i$, $V_i$ is the total variance in block $i$, the second inequality is by Cauchy–Schwarz inequality, the first equality holds due to $\sum_{i=1}^{\lceil\frac{K}{H}\rceil}B_i=B_K$, $\sum_{i=1}^{\lceil\frac{K}{H}\rceil} V_i=V_K$, the second equality holds due to $w^+<w^*<2w^+$ and $\alpha^+<\alpha^*<2\alpha^+$, the last equality holds by Corollary \ref{corollary w alpha opt}.

We then try to bound the second term (2). We denote by $\mathcal{E}$ the event such that Lemma \ref{lemma:bob} holds, and denote by $R_i:=\sum_{k=(i-1)H+1}^{iH}\la\ab_t(w^+,\alpha^+),\btheta_k\ra-\la\ab_t(w_i,\alpha_i),\btheta_k\ra$ the instantaneous regret of the meta learner in the block $i$. Then we have
\begin{align}
    (2)&=\EE\bigg[\sum_{i=1}^{\lceil\frac{K}{H}\rceil}R_i\bigg]\notag\\
    &=\EE\bigg[\sum_{i=1}^{\lceil\frac{K}{H}\rceil}R_i|\mathcal{E}\bigg]P(\mathcal{E})+\EE\bigg[\sum_{i=1}^{\lceil\frac{K}{H}\rceil}R_i|\overline{\mathcal{E}}\bigg]P(\mathcal{\overline {\mathcal{E}}})\notag\\
    &\leq\Tilde{O}\bigg(L_{\text{max}}\sqrt{\frac{K}{H}\left|\mathcal{P}\right|}\bigg)\cdot(1-\frac{2}{K})+\Tilde{O}(K)\cdot\frac{2}{K}\notag\\
    &=\Tilde{O}(\sqrt{H\left|\mathcal{P}\right|K})\notag\\
    &=\Tilde{O}(d^{\frac{1}{5}}K^{\frac{7}{10}})\,,
\end{align}
where $L_{\text{max}}:=\max_{i\in[\lceil\frac{K}{H}\rceil]} L_i$, the first inequality holds due to the standard regret upper bound result for Exp3 \cite{auer2002nonstochastic}, the third equality holds due to Lemma \ref{lemma:bob}, the last equality holds since $H=\lceil d^{\frac{2}{5}}K^{\frac{2}{5}}\rceil$, and $\left|\mathcal{P}\right|=O(\log K)$.

Finally, combining the above results for term (1) and term (2), we have
    \begin{align}
        \text{Regret}(K) = \tilde O(d^{4/5}V_K^{2/5}B_K^{1/5}K^{2/5} + d^{2/3}B_K^{1/3}K^{2/3}+d^{\frac{1}{5}}K^{\frac{7}{10}}).
    \end{align}

\section{Technical Lemmas}
\begin{theorem}[Theorem 4.3, \cite{zhou2022computationally}]\label{lemma:concentration_variance} 
Let $\{\cG_k\}_{k=1}^\infty$ be a filtration, and $\{\xb_k,\eta_k\}_{k\ge 1}$ be a stochastic process such that
$\xb_k \in \RR^d$ is $\cG_k$-measurable and $\eta_k \in \RR$ is $\cG_{k+1}$-measurable.
Let $L,\sigma,\lambda, \epsilon>0$, $\bmu^*\in \RR^d$. 
For $k\ge 1$, 
let $y_k = \la \bmu^*, \xb_k\ra + \eta_k$ and
suppose that $\eta_k, \xb_k$ also satisfy 
\begin{align}
    \EE[\eta_k|\cG_k] = 0,\ \EE [\eta_k^2|\cG_k] \leq \sigma^2,\  |\eta_k| \leq R,\,\|\xb_k\|_2 \leq L.
\end{align}
For $k\ge 1$, let $\Zb_k = \lambda\Ib + \sum_{i=1}^{k} \xb_i\xb_i^\top$, $\bbb_k = \sum_{i=1}^{k}y_i\xb_i$, $\bmu_k = \Zb_k^{-1}\bbb_k$, and
\begin{align}
    \beta_k &= 12\sqrt{\sigma^2d\log(1+kL^2/(d\lambda))\log(32(\log(R/\epsilon)+1)k^2/\delta)} \notag \\
    &\quad + 24\log(32(\log(R/\epsilon)+1)k^2/\delta)\max_{1 \leq i \leq k} \{|\eta_i|\min\{1, \|\xb_i\|_{\Zb_{i-1}^{-1}}\}\} + 6\log(32(\log(R/\epsilon)+1)k^2/\delta)\epsilon.\notag
\end{align}
Then, for any $0 <\delta<1$, we have with probability at least $1-\delta$ that, 
\begin{align}
    \forall k\geq 1,\ \big\|\textstyle{\sum}_{i=1}^{k} \xb_i \eta_i\big\|_{\Zb_k^{-1}} \leq \beta_k,\ \|\bmu_k - \bmu^*\|_{\Zb_k} \leq \beta_k + \sqrt{\lambda}\|\bmu^*\|_2.\notag
\end{align}
\end{theorem}

\begin{lemma}[Lemma 11,  \cite{abbasi2011improved}]\label{Lemma:abba}
    For any $\lambda>0$ and sequence $\{\xb_k\}_{k=1}^K \subset \RR^d$
for $k\in [K]$, define $\Zb_k = \lambda \Ib+ \sum_{i=1}^{k-1}\xb_i\xb_i^\top$.
Then, provided that $\|\xb_k\|_2 \leq L$ holds for all $k\in [K]$,
we have
\begin{align}
    \sum_{k=1}^K \min\big\{1, \|\xb_k\|_{\Zb_{k}^{-1}}^2\big\} \leq 2d\log\big(1+KL^2/(d\lambda)\big).\notag
\end{align}
\end{lemma}

% \begin{lemma}[Lemma 5, \cite{cheung2019learning}]
%     \label{lemma:sw2}
%     For any $i\leq \lceil T/w\rceil-1,$ 
%     \begin{align*}
%         \sum_{t=i\cdot w+1}^{(i+1) w}1\wedge\left\|X_t\right\|^2_{V^{-1}_{t-1}}\leq\sum_{t=i\cdot w+1}^{(i+1) w}1\wedge\left\|X_t\right\|^2_{\overline{V}_{t-1}^{-1}},
%     \end{align*}
%     where 
%     \begin{align}
%         \overline{V}_{t-1}=\sum_{s=i\cdot w+1}^{t-1}X_sX_s^{\top}+\lambda I, V_{t-1}=\lambda I+\sum_{\min\{1,t-w+1\}}^{t-1}X_sX_s^{\top}.
%     \end{align}
% \end{lemma}
\begin{theorem}[Theorem 2.1, \cite{zhao2023variance}]\label{thm:bernstein1}
    Let $\{\cG_k\}_{k=1}^\infty$ be a filtration, and $\{\xb_k,\eta_k\}_{k\ge 1}$ be a stochastic process such that
    $\xb_k \in \RR^d$ is $\cG_k$-measurable and $\eta_k \in \RR$ is $\cG_{k+1}$-measurable.
    Let $L,\sigma,\lambda, \epsilon>0$, $\bmu^*\in \RR^d$. 
    For $k\ge 1$, 
    let $y_k = \la \bmu^*, \xb_k\ra + \eta_k$, where $\eta_k, \xb_k$ satisfy 
    \begin{align}
        \EE[\eta_k|\cG_k] = 0, \ |\eta_k| \leq R, \ \sum_{i = 1}^k \EE[\eta_i^2|\cG_i] \le v_k,  \ \ \text{for}\ \forall \ k\geq 1 \notag
    \end{align}
    For $k\ge 1$, let $\Zb_k = \lambda\Ib + \sum_{i=1}^{k} \xb_i\xb_i^\top$, $\bbb_k = \sum_{i=1}^{k}y_i\xb_i$, $\bmu_k = \Zb_k^{-1}\bbb_k$, and
    %\begin{small}
    \begin{align}
        \beta_k &= 16\rho \sqrt{v_k \log(4w^2 / \delta)} + 6 \rho R  \log(4w^2 / \delta), \notag
    \end{align}
    %\end{small} 
    where $\rho \ge \sup_{k \ge 1}\|\xb_k\|_{\Zb_{k - 1}^{-1}}$. 
Then, for any $0 <\delta<1$, we have with probability at least $1-\delta$ that, 
    \begin{align}
        \forall k\geq 1,\ \big\|\textstyle{\sum}_{i=1}^{k} \xb_i \eta_i\big\|_{\Zb_k^{-1}} \leq \beta_k,  \|\bmu_k - \bmu^*\|_{\Zb_k} \leq \beta_k + \sqrt{\lambda}\|\bmu^*\|_2. \notag
    \end{align}
\end{theorem}

   \begin{lemma} [Adopted from Lemma B.4, \cite{zhao2023variance}]\label{lemma:var}
        Let weight $w_i$ be defined in Algorithm \ref{alg:1}. 
        With probability at least $1 - 2\delta$, for all $k \ge 1$, $\ell \in [L]$, the following two inequalities hold simultaneously: \begin{align*}
            \sum_{i \in \hat\Psi_{k + 1, \ell}} w_i^2 \sigma_i^2 \le 2 \sum_{i \in \hat\Psi_{k + 1, \ell}} w_i^2 \epsilon_i^2 + \frac{14}{3} R^2 \log(4w^2 L / \delta), 
            \\
            \sum_{i \in \hat\Psi_{k + 1, \ell}} w_i^2 \epsilon_i^2 \le \frac{3}{2} \sum_{i \in \hat\Psi_{k + 1, \ell}} w_i^2 \sigma_i^2 + \frac{7}{3} R^2 \log(4w^2 L / \delta). 
        \end{align*}
        %We denote the corresponding event by $\cE_{\var}$. 
    \end{lemma}
        For simplicity, we denote $\cE_{\var}$ as the event such that the two inequalities in Lemma \ref{lemma:var} holds. 
    \begin{lemma} [Adopted from Lemma B.5, \cite{zhao2023variance}]\label{lemma:varest}
Suppose that $\|\btheta^*\|_2 \le B$. Let weight $w_i$ be defined in Algorithm \ref{alg:1}. 
        On the event $\cE_{\conf}$ and $\cE_{\var}$ (defined in Eq.(\ref{event unknown 1}), Lemma \ref{lemma:var}), for all $k \ge 1$, $\ell \in [L]$ such that $2^\ell \ge 64 \sqrt{\log\left(4(w + 1)^2 L /\delta\right)}$, we have the following inequalities: 
        \begin{align*} 
            \sum_{i \in \Psi_{k + 1, \ell}} w_i^2 \sigma_i^2 \le 8 \sum_{i \in \Psi_{k + 1, \ell}} w_i^2 \left(r_i - \la \hat\btheta_{k + 1, \ell}, \ab_i \ra \right)^2 + 6R^2 \log(4(w + 1)^2 L / \delta) + 2^{-2\ell + 2}B^2, \\
            \sum_{i \in \Psi_{k + 1, \ell}} w_i^2 \left(r_i - \la \hat\btheta_{k + 1, \ell}, \ab_i \ra \right)^2 \le \frac{3}{2} \sum_{i \in \Psi_{k + 1, \ell}} w_i^2 \sigma_i^2 + \frac{7}{3} R^2 \log(4w^2 L / \delta) + 2^{-2\ell}B^2. 
        \end{align*}
    \end{lemma}
%     \begin{lemma}[Adopted from Lemma 7 in \cite{zhao2020simple}]
%   \label{lemma:bob}
%   Let $N= \lceil \frac{K}{H}\rceil$. Denote by $L_i$ the absolute value of cumulative rewards for episode $i$, i.e., $L_i =\sum_{k = (i-1)H + 1}^{iH} r_k$, then 
%   \begin{equation}
%   \label{eq:concentration}
%     \mathbb{P}\left[\forall i\in [N], L_i\leq H+2R\sqrt{H\ln\frac{K}{\sqrt{H}}}\right] \geq 1-\frac{2}{K}.
%   \end{equation}
% \end{lemma}
    \begin{lemma}[\cite{freedman1975tail}]  \label{lemma:freedman}
        Let $M, v > 0$ be fixed constants. Let $\{x_i\}_{i = 1}^n$ be a stochastic process, $\{\cG_i\}_i$ be a filtration so that for all $i \in [n]$, $x_i$ is $\cG_{i}$-measurable, while almost surely \begin{align*}
            \EE\left[x_i | \cG_{i - 1}\right] = 0, \quad |x_i| \le M, \quad \sum_{i = 1}^n \EE[x_i^2|\cG_{i-1}] \le v. 
        \end{align*}
        Then for any $\delta > 0$, with probability at least $1 - \delta$, we have \begin{align*} 
            \sum_{i = 1}^n x_i \le \sqrt{2v \log(1 / \delta)} + 2 / 3 \cdot M \log(1 / \delta). 
        \end{align*}
    \end{lemma}
    \begin{lemma}
  \label{lemma:bob}
  Let $N= \lceil \frac{K}{H}\rceil$. Denote by $L_i$ the absolute value of cumulative rewards for episode $i$, i.e., $L_i =\sum_{k = (i-1)H + 1}^{iH} r_k$, then 
  \begin{equation}
  \label{eq:concentration}
    \mathbb{P}\left[\forall i\in [N], L_i\leq H+R\sqrt{\frac{H}{2} \log\big(K(\frac{K}{H}+1)\big)} + \frac{2}{3} \cdot R \log\big(K(\frac{K}{H}+1)\big)\right] \geq 1-\frac{1}{K}.
  \end{equation}
\end{lemma}
\begin{proof}
    By Lemma \ref{lemma:freedman}, we have that with probability at least $1-1/K$
    \begin{align}
        \sum_{k = (i-1)\cdot H+1}^{i\cdot H} \epsilon_i &\le \sqrt{2\sum_{k = (i-1)\cdot H+1}^{i\cdot H}\sigma_k^2 \log(NK)} + 2 / 3 \cdot R \log(NK)\notag\\
        &\leq \sqrt{2H\frac{R^2}{4} \log(NK)} + 2 / 3 \cdot R \log(NK)\notag\\
        &\leq R\sqrt{\frac{H}{2} \log\big(K\cdot(\frac{K}{H}+1)\big)} + \frac{2}{3} \cdot R \log\big(K\cdot(\frac{K}{H}+1)\big)\,,
    \end{align}
    where we use union bound, and in the second inequality we use the fact that since $\left|\epsilon_k\right|\leq R$, we have $\sigma_k^2\leq \frac{R^2}{4}$. Finally, together with the assumption that $r_k\leq 1$ for all $k\in[K]$, we complete the proof. 
\end{proof}

\bibliography{refs}
\bibliographystyle{ims}

\end{document}